\documentclass[lettersize,journal]{IEEEtran}
\usepackage{amsmath,amsfonts,amsthm,amsbsy, amssymb}
\usepackage{algorithmic}
\usepackage{algorithm}
\usepackage{array}
\usepackage{textcomp}
\usepackage{stfloats}
\usepackage{float}
\usepackage{url}
\usepackage{verbatim}
\usepackage{graphicx}
\usepackage{cite}
\usepackage{hyperref}
\usepackage{titlesec}
\usepackage{bm}
\usepackage{xcolor}
\usepackage{tabularray}
\usepackage{changepage}

\usepackage{titlesec}
\newcommand{\EdgeSet}{\vec{\bm{\mathcal{E}}}}
\newcommand{\VertSetP}{\bm{\mathcal{V}}_p}
\newcommand{\VertSetM}{\bm{\mathcal{V}}_m}
\newcommand{\VertSet}{\bm{\mathcal{V}}}
\newcommand{\Graph}{\vec{\bm{\mathcal{G}}}}

\newcommand{\vect}[1]{\mbox{vec}(#1)}
\newcommand{\inner}[2]{\left<#1,#2\right>}
\newcommand{\tr}[1]{\mbox{tr}\left(#1\right)}
\newcommand{\rank}[1]{\mbox{rank}\left(#1\right)}

\newcommand{\diag}[1]{\mbox{diag}\left(#1\right)}
\newcommand{\indset}[1]{\left[#1\right]}
\newtheorem{theorem}{Theorem}
\newtheorem{lemma}[theorem]{Lemma}
\newtheorem{definition}[theorem]{Definition}

\newtheorem{proposition}[theorem]{Proposition}
\newtheorem{remark}{Remark}

\newcommand{\rev}[1]{{\color{black}{#1}}}
\begin{document}

\title{On Semidefinite Relaxations for Matrix-Weighted State-Estimation Problems in Robotics}

\author{Connor Holmes~\IEEEmembership{Student Member, IEEE}, Frederike D{\"u}mbgen~\IEEEmembership{Member, IEEE}, Timothy D. Barfoot~\IEEEmembership{Fellow, IEEE}\vspace*{-0.45in}
% <-this % stops a space
%\thanks{Manuscript received April 19, 2021; revised August 16, 2021.}
%<-this % stops a space   
\thanks{This work was supported in part by the National Sciences and Engineering Research Council of Canada (NSERC) and in part by Swiss National Science Foundation, Postdoc Mobility under Grant 206954.}
\thanks{CH and TDB are with the University of Toronto Robotics Institute, University of Toronto, Toronto, Ontario, Canada. FD is with Inria, École Normale Supérieure, PSL University, Paris, France. Corresponding author: {\tt\footnotesize connor.holmes@mail.utoronto.ca}}
\thanks{Manuscript received: August 14, 2023; Revised: May 1, 2024; Accepted: September 22, 2024.}
\thanks{This paper was recommended for publication by Editor Behnke, Sven upon evaluation of the Associate Editor and Reviewers’ comments.}
\thanks{Digital Object Identifier (DOI): see top of this page.}
}
\markboth{Revision, September~2024}%
{Shell \MakeLowercase{\textit{et al.}}: A Sample Article Using IEEEtran.cls for IEEE Journals}


% REPLIES GO HERE. COMMENT OUT FOR CAMERA-READY VERSION
% \input{replies_cond_accept.tex}

% Start paper
\maketitle

\begin{abstract}
In recent years, there has been remarkable progress in the development of so-called \emph{certifiable perception} methods, which leverage semidefinite, convex relaxations to find \emph{global optima} of perception problems in robotics. However, many of these relaxations rely on simplifying assumptions that facilitate the problem formulation, such as an \emph{isotropic} measurement noise distribution.
In this paper, we explore the tightness of the semidefinite relaxations of \emph{matrix-weighted} (anisotropic) state-estimation problems and reveal the limitations lurking therein: matrix-weighted factors can cause convex relaxations to lose tightness. In particular, we show that the semidefinite relaxations of localization problems with matrix weights may be tight only for low noise levels. 
To better understand this issue, we introduce a theoretical connection between the posterior uncertainty of the state estimate and \rev{the certificate matrix obtained via} convex relaxation.  
With this connection in mind, we empirically explore the factors that contribute to this loss of tightness and demonstrate that \emph{redundant constraints} can be used to regain it.
As a second technical contribution of this paper, we show that the state-of-the-art relaxation of scalar-weighted SLAM cannot be used when matrix weights are considered. We provide an alternate formulation and show that its SDP relaxation is not tight (even for very low noise levels) unless specific \emph{redundant constraints} are used. 
We demonstrate the tightness of our formulations on both simulated and real-world data.

\end{abstract}


\section{Introduction}
\begin{figure}[!t]
	\centering
	\vspace*{-0.05in}
	\includegraphics[width=\columnwidth]{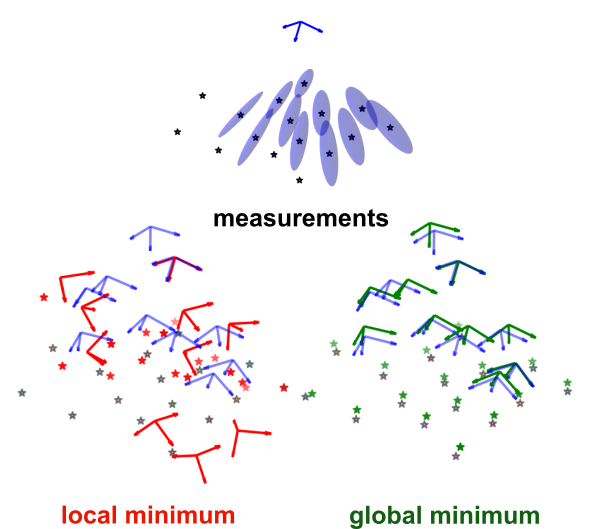}
	\vspace*{-0.3in}
	\caption{An example of a local and global minimum for stereo SLAM with 10 poses and 20 landmarks from the ``Starry Night'' dataset. Ground-truth landmarks are shown as grey stars, ground-truth poses as blue frames, pose/landmark estimates are shown as frames/stars coloured red and green for local and global minima, respectively. 
	Both minima are based on the same set of measurements: stereo pixel coordinates converted to Euclidean coordinates (example shown at the top of the figure) and relative-pose measurements.
	}
	\vspace*{-0.22in}
	\label{fig:slam_local_min}
\end{figure}

\IEEEPARstart{S}{tate-estimation} is an integral component of modern robotics systems. Workhorse algorithms for state-estimation -- such as localization and simultaneous localization and mapping (SLAM) -- are now capable of estimating hundreds of thousands of states on a single processor in real time~\cite{rosenAdvancesInferenceRepresentation2021} and are far from the computational bottleneck of robotic systems. To obtain such levels of performance, these algorithms typically rely on local optimization methods (e.g., Gauss-Newton), which often exhibit super-linear convergence. In particular, SLAM has reached a high level of maturity in terms of both breadth and depth of understanding in the robotics community (see~\cite{baileySimultaneousLocalizationMapping2006} and~\cite{durrant-whyteSimultaneousLocalizationMapping2006} for a comprehensive review of SLAM).  

In recent years, we have seen a surge in interest in the use of convex, semidefinite relaxations to solve and certify the global optimality of robotic state-estimation and perception problems. 
In principle, these relaxations can be solved in polynomial time using interior-point methods\cite{vandenbergheSemidefiniteProgramming1996}. However, much of the interest in these algorithms is related to recent improvements in runtime, which have made certifiable methods more attractive for \emph{real-time} applications. These improvements have largely been brought to the robotics community by SE-Sync~\cite{rosenSESyncCertifiablyCorrect2019}, which globally solves \emph{pose-graph optimization} (PGO) -- a cornerstone of modern SLAM algorithms -- by leveraging the low-rank nature of its semidefinite program (SDP) relaxation. A series of extensions to this method have been and continue to be developed~\cite{rosenAdvancesInferenceRepresentation2021}.

Despite the success of these methods, they often rely on simplifying assumptions in order to apply these fast, low-rank algorithms. Perhaps most striking is the assumption of isotropic noise models that pervades the majority of the literature. Isotropic noise is often an unrealistic assumption for modern robotics, especially when more realistic sensor models are considered. For example,~\cite{matthiesErrorModelingStereo1987} shows that conversion of stereo pixel coordinates to Euclidean coordinates yields noise distributions that should not be modeled isotropically. When the correct model is used, the resulting maximum-likelihood optimization includes \emph{matrix} (rather than \emph{scalar}) weighting factors. 

The introduction of matrix weighting in state-estimation problems with rotations typically leads to solution methods that are more involved. For example, though a closed-form, global solution exists for scalar-weighted Wahba's problem~\cite{hornClosedformSolutionAbsolute1987, hornClosedFormSolutionAbsolute1988, wahbaLeastSquaresEstimate1965},\footnote{Wahba's problem can also be referred to as point-set registration.} when matrix weights are introduced, an iterative, local solver must be used~\cite{chengTotalLeastSquaresEstimate2019, barfootStateEstimationRobotics2017}. %Moreover, the introduction of matrix weights could lead to a cost function landscape with more local minima. 

We will see that introduction of matrix weighting can also have a profound effect on the convex relaxations of state-estimation problems. In many cases, problems that have tight relaxations in the scalar-weighted case have a duality gap when matrix weighting is used. These relaxations can be \emph{tightened} by adding new constraints to the SDP, but the addition of these constraints may degrade the performance of  low-rank optimization methods mentioned above. Therefore, it is paramount to understand the key causes of this loss of tightness.

In this paper, we explore the tightness of semidefinite relaxations of perception problems that have been generalized to include \emph{matrix weights} and expose the limitations that result from this generalization. In the next section, we introduce works that are closely related to this paper and establish our contributions to the field.  We then introduce the requisite background on measurement models and semidefinite relaxation methods in Section~\ref{sec:Background}. 
In Section~\ref{sec:Formulations}, we explore the impact of matrix weighting on the formulation of two key state-estimation problems in robotics: localization and SLAM.
%In particular, we show that the introduction of matrix weighting necessarily changes the formulation of the semidefinite relaxation (compared to the scalar-weighted version~\cite{holmesEfficientGlobalOptimality2023, iglesiasGlobalOptimalityPoint2020}). 
We then draw an interesting theoretical connection in Section~\ref{sec:Uncertainty} between the posterior distribution of the state estimate and dual solution (or \emph{certificate matrix}) of the semidefinite relaxation of the corresponding state estimation problem. 
In Section~\ref{sec:Simulations}, we provide an in-depth, empirical study of the effects of matrix weighting, anisotropic noise, and stereo-camera measurements on the tightness of the semidefinite relaxations defined in  Section~\ref{sec:Formulations}, drawing connections to Section~\ref{sec:Uncertainty}. We also evaluate these relaxations on real-world datasets, both with and without redundant constraints, in Section~\ref{sec:RealExp}. In particular, we show how our globally optimal solution to matrix-weighted Wahba's Problem can be used in a state-of-the-art, outdoor, stereo-localization pipeline.  Finally, Section~\ref{sec:Conclusions} presents our conclusions and ideas for future research in this area.

\section{Related Work and Contributions}\label{sec:RelatedContrib}

There is a large range of problems for which certifiable methods already exist. To name a few, methods for robust state estimation~\cite{yangCertifiablyOptimalOutlierRobust2023,yangTEASERFastCertifiable2021}, sensor calibration~\cite{wiseCertifiablyOptimalMonocular2020,giamouConvexIterationDistanceGeometric2022}, inverse kinematics~\cite{giamouConvexIterationDistanceGeometric2022}, image segmentation~\cite{huAcceleratedInferenceMarkov2019}, pose-graph optimization~\cite{rosenSESyncCertifiablyCorrect2019}, multiple-point-set registration~\cite{chaudhuryGlobalRegistrationMultiple2015, iglesiasGlobalOptimalityPoint2020}, range-only localization~\cite{dumbgenSafeSmoothCertified2023}, planar SLAM~\cite{liuConvexOptimizationBased2012}, and range-aided SLAM~\cite{papaliaCertifiablyCorrectRangeAided2023} have all been explored. 

Many papers have considered the conditions under which a given problem \emph{can} be certified. In particular,~\cite{cifuentesLocalStabilitySemidefinite2022} shows that, under certain technical conditions, problems that have zero duality gap when unperturbed (no noise) continue to enjoy zero duality gap as long as the perturbation parameter is within a bound (i.e., the underlying problem has sufficiently low noise). For state-estimation problems, this bound is often (but not always) found empirically to be larger than noise levels encountered in practice~\cite{rosenSESyncCertifiablyCorrect2019, tianDistributedCertifiablyCorrect2021, erikssonRotationAveragingStrong2018}.

At present, certifiable perception problems can be broadly catagorized into two key groups: problems for which fast, low-rank solvers are available and problems that can be certified, but must still rely on slower SDP solution methods (e.g., interior-point methods). The next subsections provide more detail on each of these two groups.

\subsection{Fast Certifiable Perception}\label{sec:FastPerception}

More so than other problems, global optimization of \emph{rotation synchronization} has been the subject of intense study in the vision community~\cite{wilsonWhenRotationsAveraging2016, erikssonRotationAveragingStrong2018, brynteTightnessSemidefiniteRelaxations2022} and was one of the first to enjoy significant speed improvements by leveraging the \emph{low-rank} structure of the SDP relaxation via the so-called \emph{Riemannian Staircase} approach\cite{bandeiraTightnessMaximumLikelihood2017, boumalNonconvexBurerMonteiro2016}. 

Building off existing certification methods for PGO~\cite{carloneLagrangianDuality3D2015} and inspired by the success of Riemannian methods for rotation-synchronization,~\cite{rosenSESyncCertifiablyCorrect2019} introduced SE-Sync in the robotics community. SE-Sync solves PGO over $\mbox{SE}(d)$ by taking advantage of its \emph{separable} structure~\cite{khosoussiSparseSeparableSLAM2016}, marginalizing the translation variables, and using the \emph{Riemannian Staircase} to solve the resulting rotation-synchronization problem. It was later shown that this technique could be used without marginalizing translations~\cite{brialesCartanSyncFastGlobal2017} and extended to landmark-based SLAM~\cite{holmesEfficientGlobalOptimality2023}, in both cases further exploiting problem sparsity. This method was also extended to a \emph{distributed} framework in~\cite{tianDistributedCertifiablyCorrect2021} and has even been integrated into the recent Kimera-Multi pipeline~\cite{tianKimeraMultiRobustDistributed2022}. Not surprisingly, these developments have inspired further advances in the original rotational synchronization problem~\cite{dellaertShonanRotationAveraging2020}. 

Some of these methods boast runtimes that even rival state-of-the-art, local methods (e.g., Gauss-Newton-based methods~\cite{gtsam}), with the added guarantee of a global certificate~\cite{juricComparisonGraphOptimization2021, brialesCartanSyncFastGlobal2017}. An excellent review of the current state of certifiable methods is provided in~\cite{rosenAdvancesInferenceRepresentation2021}.

\subsection{Certifiable Perception with Redundant Constraints}\label{sec:RedundantConstraints}

Though SE-Sync and its derivatives are among the most performant algorithms in certifiable perception, there are several other certifiable perception problems for which low-rank methods do not result in significant performance improvements. In particular, this set of problems includes those whose semidefinite relaxations are not initially tight, but can be tightened through the addition of certain constraints. These constraints are referred as \emph{redundant constraints} because they are redundant in the original formulation of the problem, but become nonredundant (and, indeed, linearly independent) in the lifted, semidefinite relaxation.

This `tightening trick' has been known for some time in the optimization community~\cite{nesterovSemidefiniteProgrammingRelaxations2000}, and has been applied to several problems~\cite{ruizUsingRedundancyStrengthen2011, parriloSemidefiniteProgrammingRelaxations2003a}. In the computer-vision community, redundant constraints have been used to tighten generalized essential-matrix estimation~\cite{zhaoCertifiablyGloballyOptimal2020}, relative-pose estimation between cameras~\cite{garcia-salgueroTighterRelaxationRelative2022, brialesCertifiablyGloballyOptimal2018}, and registration using 3D primitives~\cite{brialesConvexGlobal3D2017}. The formulation given in the latter paper is a degenerate version of our single-pose, matrix-weighted localization formulation (see Section~\ref{sec:Localization}), in which the use of (degenerate) matrix weights is motivated by geometry rather than noise distribution. 

Interestingly, both~\cite{garcia-salgueroTighterRelaxationRelative2022} and~\cite{brialesCertifiablyGloballyOptimal2018} showed that adding redundant constraints increases the level of measurement noise for which their respective problems remain tight. In the context of robotics,~\cite{wiseCertifiablyOptimalMonocular2020} found a similar result when exploring the effect of redundant constraints on a sensor-calibration task.~\cite{yangTEASERFastCertifiable2021} introduced redundant constraints in conjunction with \emph{graduated non-convexity}~\cite{yangGraduatedNonConvexityRobust2020} to solve \emph{robust} point-set registration globally.~\cite{yangOneRingRule2020} and~\cite{yangCertifiablyOptimalOutlierRobust2023} extended these results to several other robust perception problems by leveraging the Lasserre-moment hierarchy~\cite{henrionMomentSOSHierarchyLectures2021,lasserreGlobalOptimizationPolynomials2001}. 

This hierarchy constitutes a powerful set of theoretical tools that are guaranteed to tighten the SDP relaxation of any \emph{polynomial} optimization problem through the use of redundant variables and constraints.\footnote{Note that this class of problems encompasses almost all of the certifiable perception problems to date.} The caveat to this method is that it requires the addition of new variables and constraints -- possibly \emph{ad infinitum} -- and can quickly become intractable in practice. Techniques such as \emph{Douglas-Rachford Splitting} for certification~\cite{yangTEASERFastCertifiable2021} and the STRIDE algorithm~\cite{yangCertifiablyOptimalOutlierRobust2023} have improved runtimes when the moment hierarchy is used, but are still far from real time. 

Since certification generally depends on the number of variables and constraints used, a more efficient tightening approach is search for a small  subset of variables and constraints that is sufficient  to render  a given relaxation \emph{tight}. This reflects the approach that we take in this paper as well as in our concurrent work~\cite{dumbgenGloballyOptimalState2023a}, which searches for redundant constraints using a sampling approach.

\subsection{Our Contributions}

The contributions of this paper are as follows:
\begin{itemize}
	\item We demonstrate that introducing matrix weights (due to anisotropic noise distributions) into existing certifiable state estimation problems can severly impact the tightness of the underlying SDP relaxation.
	\item We show that a set of redundant constraints can be used to regain tightness in these problems. To do this, we leverage results in our concurrent paper~\cite{dumbgenGloballyOptimalState2023a}, which \emph{numerically} find a redundant constraint set for a \emph{specific} problem. We interpret these numerical constraints to find algebraic constraints that can be applied to a broader range of problem instances.
	\item We establish a connection between classical probabilistic interpretations of uncertainty and the \emph{dual} or \emph{certificate} matrix and leverage this connection to further understand our empirical results and the effect of redundant constraints on tightness.
	\item We show that, while the SDP relaxation of matrix-weighted SLAM is intractable for large-scale problems, the relaxation of matrix-weighted Wahba's problem can be solved in near realtime. We apply the latter relaxation in an outdoor, stereo-localization pipeline on real-world data.
\end{itemize}

\section{Background Material}\label{sec:Background}

In this section, we provide the reader with some background material and notation that will be useful for the understanding of subsequent developments.

\subsection{Notation}\label{sec:Notn}
We denote matrices with bold-faced, capitalized letters, $ \bm{A} $, column vectors with bold-faced, lower-case letters, $ \bm{a} $, and scalar quantities with normal-faced font, $ a $.
Let $ \mathbb{S}^n $ denote the space of $ n $-dimensional symmetric matrices and $ \mathbb{S}_+^n $ denote the space of $ n $-dimensional symmetric positive semidefinite matrices. We equivalently write $\bm{A}\succeq\bm{0}$ whenever $\bm{A}\in\mathbb{S}_+^n$. 
Let $ \|\cdot\|_F $ denote the Frobenius norm and let $\left<\bm{A},\bm{B}\right>$ denote the Frobenius inner product. 
Let $ \diag{\bm{A}_1,\cdots,\bm{A}_N} $ denote the block-diagonal matrix with blocks corresponding to matrices $ \bm{A}_1,\cdots,\bm{A}_N $. Note that this includes the case where the $ \bm{A}_i $ are scalar (i.e., $\diag{a_1,\cdots,a_N}$, $a_i \in \mathbb{R}$).
Let $ \bm{I} $ denote the identity matrix, whose dimension will be clear from the context or otherwise specified.
Let $ \bm{0} $ denote the matrix with all-zero entries, whose dimension will be evident from the context.
Let the subscript ``$ 0 $'' denote the world frame.
Let $ \bm{t}_i^{ji} $ denote a vector from frame $ i $ to frame $ j $ expressed in frame $ i $ and $ \bm{C}_{ij} $ denote a rotation matrix that maps vectors expressed in frame $ j $ to equivalent vectors in frame $ i $. For readability, we replace $ \bm{t}_i^{i0} $ with $ \bm{t}_i $ and $\bm{C}_{i0}$ with $\bm{C}_i$. 
Let $ \otimes $ denote the Kronecker product.
Let $ \vert S \vert $ denote the cardinality of the set $ S $.
Let $ \bm{A}^+ $ denote the Moore-Penrose pseudoinverse of a given matrix $ \bm{A} $.
Let $ \vect{\bm{A}} $ denote the column-wise  vectorization (reshape) of a given matrix $ \bm{A} $.
Let $ (\cdot)^{\times} $ denote the linear, skew-symmetric operator as defined in~\cite{barfootStateEstimationRobotics2017}.
Let $\left[N\right] = \left\{1,\cdots, N\right\} \subset \mathbb{N}$ be the set of indexing integers.

\subsection{MAP Estimation and the Fisher Information Matrix}\label{sec:MeasModels}

In robotics, we often frame state estimation as \emph{maximum-a-posteriori} (MAP) problems, in which the optimal estimate is given by
\begin{equation}\label{opt:ConstrainedMAP}
	\bm{x}_c^* = \arg\min\limits_{\bm{x}_c\in\mathcal{M}} -\log\left(p(\bm{x}_c \vert \bm{\mathcal{D}})\right),
\end{equation}
where $p(\bm{x}_c \vert \bm{\mathcal{D}})$ is the posterior distribution function of the estimated parameter, $\bm{x}_c$, given all of the available data, $\bm{\mathcal{D}}$ and $\mathcal{M}$ characterizes the feasible set. In practice, it is common to approximate Problem \eqref{opt:ConstrainedMAP} as follows:
\begin{equation}\label{opt:UnconstrainedMAP}
	\bm{x}^*=\arg\min\limits_{\bm{x}\in\mathbb{R}^p} -\log\left(p(\bm{x} \vert \bm{\mathcal{D}})\right),
\end{equation}
where the optimization variable, $\bm{x}$, has been locally parameterized to ensure that constraints are not explicitly required.\footnote{When state estimates involve pose variables or orientations this unconstrained form can be derived using the Lie algebra vector space of $\mbox{SE}(3)$ or $\mbox{SO}(3)$.}
Optimization then proceeds by iteratively updating the local parameterization until convergence is reached.

It is often the case that we wish to ascertain not only optimal estimates, but also the uncertainty associated with these estimates. To do so, we use the \emph{Laplace approximation}, which models the posterior distribution as a Gaussian centered at the MAP estimate, $\bm{x}^*$, with inverse covariance equal to the so-called \emph{Fisher Information Matrix} (FIM), 
\begin{equation}
	\bm{\Sigma}^{-1} = -\left.\frac{\partial^2}{\partial\bm{x}^2}\log\left(p(\bm{x} \vert \bm{\mathcal{D}})\right)\right|_{\bm{x}=\bm{x}^*_u}.
\end{equation}
The FIM can be extracted directly from the Hessian of the cost function in \eqref{opt:UnconstrainedMAP} (or an approximation thereof)\footnote{When using Gauss-Newton to solve estimation problems, the Hessian is often approximated as the product of the Jacobian with its transpose.}. Its properties have been intensely studied by the robotic state-estimation community~\cite{censiAchievableAccuracyPose2009}. For instance, it is known that the minimum eigenvalues of the FIM (and their respective eigenvectors) characterize the worst-case uncertainty of an estimated parameter. 
Of particular interest here is the fact that the geometry of the measurement data (e.g., aligned uncertainty ellipsoids) can lead to degeneracy of the FIM~\cite{zhangDegeneracyOptimizationbasedState2016}.

% \begin{equation}
% 	\log\left(p(\bm{x} \vert \bm{\mathcal{D}})\right) \approx \eta + (\bm{x}-\bm{x}^*)^T\bm{\Sigma}^{-1}(\bm{x}-\bm{x}^*)
% \end{equation}
\subsection{Measurement Models}

In this paper, we define a \textit{directed graph}, $ \Graph = \left(\VertSet, \EdgeSet\right)$, to keep track of poses  and measurements. The vertex set $ \VertSet = \VertSetP \cup \VertSetM $ is the union of the set of vertices representing poses, $ \VertSetP =\left[N_p\right]$, and the set of vertices representing landmarks, $ \VertSetM = \left[N_m\right]$. We assume that the edge set $\EdgeSet \subset \VertSet \times \VertSet$ is partitioned as $\EdgeSet= \EdgeSet_p \cup \EdgeSet_m$, where $\EdgeSet_p\subset \VertSet_p\times \VertSet_p$ represents relative-pose measurements and $\EdgeSet_m\subset \VertSet_p\times \VertSet_m$ represents measurements of a landmark from a given pose. We assume that each edge, $(i,j)\in \EdgeSet$, is associated with an error term, $\bm{e}_{ij}$, and a matrix weight, $\bm{W}_{ij}$. The $i^{th}$ pose is a member of the Special Euclidean group:
\begin{equation}\label{eqn:pose_var}
	\mbox{SE}(3) = \left\{(\bm{C}_i, \bm{t}_i)~:~\bm{C}_i\in \mbox{SO}(3),~ \bm{t}_i\in \mathbb{R}^3\right\}.
\end{equation}
$\bm{C}_i$ represents the rotation from the world frame to the $i^{th}$ pose frame and $\bm{t}_i$ represents the translation vector from the world frame to the pose frame, expressed in the pose frame. Moreover, we will make use of the following \textit{homogeneous transformation} to represent a given robot pose:
\begin{equation}
	\bm{T}_i = \begin{bmatrix}
		\bm{C}_i & -\bm{t}_i \\ \bm{0} & 1
	\end{bmatrix}.
\end{equation}

Here, we consider MAP optimization problems over pose and landmark variables in which the cost (log-posterior) can be expressed in the following factored form\cite{barfootStateEstimationRobotics2017}:
\begin{equation}\label{eqn:factor_graph}
	\begin{gathered}
		-\log\left(p(\bm{x} \vert \bm{\mathcal{D}})\right) = \sum\limits_{(i,j)\in\EdgeSet_p}  J_{ij}^{p} + \sum\limits_{(i,k)\in\EdgeSet_m} J_{ik}^{m}, \\ 
		J_{ij}^{p} = \bm{e}_{ij}^T \bm{W}_{ij} \bm{e}_{ij}, ~J_{ik}^{m} = \bm{e}_{ik}^T \bm{W}_{ik} \bm{e}_{ik},
	\end{gathered}
\end{equation}
where $J_{ij}^{p}$ and $J_{ik}^{m}$ are the cost `factors', with error terms, $\bm{e}_{ij}$ and $\bm{e}_{ik}$, that depend on the problem variables and weighting matrices, $\bm{W}_{ij}$ and $\bm{W}_{ik}$. 
The exact form of these terms are discussed in the next two sections. Though our formulation here is matrix-weighted in general, we focus on \emph{anisotropic} noise for pose-landmark measurements, with relative-pose measurements remaining \emph{isotropic}.

\subsubsection{Matrix-Weighted Pose-Landmark Measurements}\label{sec:LandmarkMeas}

Each edge, $(i,k) \in \EdgeSet_m $, represents a measurement of the $k^{th}$ landmark from the $i^{th}$ pose. In robotics, the sensors that provide measurements of landmarks can often be modeled as
\begin{equation}
	\bm{d}_{ik} = \bm{g}(\bm{C}_i\bm{m}_{0}^{k0} - \bm{t}_i) + \bm{\epsilon}_{d,ik}, \quad \bm{\epsilon}_{d,ik} \in \mathcal{N}(\bm{0},\bm{\Sigma}_{d,ik}),
\end{equation}
where $ \bm{d}_{ik} $ represents the (raw) measurement of landmark $k$ from the $i^{th}$ pose variable, $\bm{m}_{0}^{k0}$ is the location of landmark $k$ in the global frame, $ \bm{g}(\cdot) $ is an invertible, smooth, non-linear function, and $ \bm{\epsilon}_{d,ik} $ is a zero-mean error term having normal distribution with associated covariance matrix, $ \bm{\Sigma}_{d,ik} $. 

It is often desirable to convert measurements $ \bm{d}_{ik} $ to a more convenient form by inverting the measurement model (if possible) and defining the pseudo-measurement, $\tilde{\bm{m}}_i^{ki} = \bm{g}^{-1}(\bm{d}_{ik})$, with \rev{approximate} mean given by
\begin{equation*}
	\mathbb{E}(\tilde{\bm{m}}_i^{ki})~\rev{\approx}~\bm{C}_i\bm{m}_{0}^{k0} - \bm{t}_i.
\end{equation*} 
The deviation from the mean, given by
\begin{equation}\label{eqn:lm_error_term}
	\bm{e}_{ik} = \tilde{\bm{m}}_i^{ki} - (\bm{C}_i\bm{m}_{0}^{k0} - \bm{t}_i),
\end{equation}
is \emph{approximately} Gaussian with zero mean and is exactly the error term for this measurement factor. Regardless of whether $ \bm{\Sigma}_{d,ik} $ represents isotropic noise, the (linearly-transformed) covariance of $\bm{\epsilon}_{ik} $ is typically \textit{anisotropic} and is given by
\begin{equation}
	\bm{\Sigma}_{ik} = \bm{G}^T \bm{\Sigma}_{d,ik} \bm{G},
\end{equation}
where $\bm{G}$ is the Jacobian of the inverse measurement function $\bm{g}^{-1}(\cdot)$~\cite{matthiesErrorModelingStereo1987}. 
The weighting matrix in the cost factor in the MAP estimation, \eqref{eqn:factor_graph}, is then given by the inverse of the covariance matrix, $ \bm{W}_{ik} = \bm{\Sigma}^{-1}_{ik} $. Note that solving the problem under the simplifying assumption of isotropic noise -- that is, $ \bm{W}_{ik} = \sigma_{ik} \bm{I}$ with $ \sigma_k \in \mathbb{R}$ -- can be extremely detrimental to the quality of the final solution~\cite{matthiesErrorModelingStereo1987,maimoneTwoYearsVisual2007}. 

A key example of such a situation arises in a common preprocessing step of stereo-vision problems, in which pixel and disparity measurements are converted to Euclidean point measurements by inverting a known camera model.  It has been shown that this results in measurement uncertainty that is much larger in the depth direction of a given camera pose~\cite{matthiesErrorModelingStereo1987,barfootPoseEstimationUsing2011} (c.f. Figure~\ref{fig:slam_local_min}). A derivation of the inverse stereo measurement model and associated covariance is provided in Appendix~\ref{App:stereo}. 

\subsubsection{Relative-Pose Measurements}\label{sec:RelPoseMeas}

Each edge, $(i,j) \in \EdgeSet_p $, represents a relative-pose measurement, $(\tilde{\bm{C}}_{ij},\tilde{\bm{t}}_i^{ji}) \in \mbox{SE}(3)$ and its associated homogeneous transformation, $\tilde{\bm{T}}_{ij}$.  In a robotics context, these relative-pose measurements often represent IMU-based measurement information, dynamics-based prior belief propagation, or aggregate measurement data between keyframes. Similarly to~\cite{brialesCartanSyncFastGlobal2017}, we define the following relative-pose error term:
\begin{equation}\label{eqn:rel_pose_err}
	\bm{e}_{ij} = \vect{\tilde{\bm{T}}_{ij}\bm{T}_j - \bm{T}_i}.
\end{equation}
Though $\bm{W}_{ij} \succeq \bm{0} $ can be an \emph{arbitrary}, positive semidefinite matrix, we select weight matrices of the form $\bm{W}_{ij}=\diag{\sigma^2_{ij},\sigma^2_{ij},\sigma^2_{ij},\tau^2_{ij}}^{-1} \otimes \bm{I}$, since it allows us to express our cost factor as
\begin{align}\label{eqn:rel_pose_cost}
	J_{ij} &= \frac{1}{\sigma^2_{ij}} \left\Vert \tilde{\bm{C}}_{ij}\bm{C}_{j0} - \bm{C}_i\right\Vert_F^2 + \frac{1}{\tau^2_{ij}} \left\Vert \tilde{\bm{t}}^{ji}_{i} - \tilde{\bm{C}}_{ij}\bm{t}_j + \bm{t}_i \right\Vert_2^2 \nonumber\\
	&\hspace*{3pt}=\bm{e}_{ij}^T \bm{W}_{ij} \bm{e}_{ij}.
\end{align}
where $\sigma_{ij}$ and $\tau_{ij}$ are scalar weights.\footnote{Note that these weights represent isotropic noise, although representing other distributions may be possible.} This form of cost factor is similar to the cost used in pose-graph optimization problems (c.f.~\cite{brialesCartanSyncFastGlobal2017, rosenSESyncCertifiablyCorrect2019, holmesEfficientGlobalOptimality2023}) and is quadratic for our choice of pose variables, \eqref{eqn:pose_var}.\footnote{The choice to represent translation variables in the \emph{local} frame ($\bm{t}_i$) rather in the \emph{world} frame ($\bm{t}_0^{i0}$) allows us to keep the cost function quadratic in the variables.}

The relative-pose error formulation also allows for the addition of \emph{prior} information about the pose variables to \eqref{eqn:factor_graph}. In this case, the relative-pose error is defined with respect to the world frame:
\begin{equation}\label{eqn:prior_pose_err}
	\bm{e}_{0j} = \vect{\tilde{\bm{T}}_{0j}\bm{T}_j - \bm{I}}.
\end{equation}

\subsection{Convex Relaxations of QCQPs}\label{sec:Relaxation}

We review the well-known procedure for deriving convex, SDP relaxations of a standard form of polynomial optimization problem.\footnote{For the sake of brevity, our introduction of these problems and concepts in this section have been stated without proof. The interested reader is referred to~\cite{cifuentesLocalStabilitySemidefinite2022,vandenbergheSemidefiniteProgramming1996} and references therein for more extensive expositions.} This procedure was pioneered in~\cite{shorQuadraticOptimizationProblems1987} and has become the cornerstone of \textit{certifiably correct methods} in robotics and computer vision~\cite{rosenAdvancesInferenceRepresentation2021,brynteTightnessSemidefiniteRelaxations2022}. Here, we consider a \textit{homogeneous}, \textit{quadratically constrained quadratic problem} (QCQP) expressed in the following standard form:
\begin{equation}\label{opt:QCQP}
	\min\limits_{\bm{z}} \left\{ \bm{z}^T\bm{Q}\bm{z} ~\vert~ \bm{z}^T\bm{A}_i\bm{z} = 0,\forall i \in \indset{N_c}, \bm{z}^T\bm{A}_0\bm{z} = 1 \right\},
\end{equation}
where $ \bm{z}\in\mathbb{R}^n $ is the homogeneous optimization variable, $ \bm{Q}\in\mathbb{R}^{n\times n} $ represents the quadratic cost, the $ \bm{A}_i $ correspond to $ N_c $ quadratic constraints, and $ \bm{A}_0 $ corresponds to the so-called \emph{homogenizing constraint}\cite{cifuentesLocalStabilitySemidefinite2022}. Note that any problem with quadratic cost and constraints can be converted into a problem of this form~\cite{cifuentesLocalStabilitySemidefinite2022}. This includes cost functions of the form given in \eqref{eqn:factor_graph}, as long as the error terms are \emph{linear} in the optimization variables and the constraints are quadratic.

In general, problems with quadratic equality constraints, such as QCQPs, are difficult to solve optimally because they are non-convex~\cite{boydConvexOptimization2004}. However, a popular approach to obtaining globally optimal solutions to QCQPs involves formulating the convex, SDP  relaxation of the QCQP and subsequently showing that this relaxation is \emph{tight}. In this context, `tight' means that the relaxation and the original problem have the same optimal cost and equivalent minimizers.

It can be shown that Problem \eqref{opt:QCQP} is equivalent to the following problem~\cite{cifuentesLocalStabilitySemidefinite2022}:
\begin{equation}\label{opt:R1SDP}
	\begin{array}{rl}
		\min\limits_{\bm{X}} & \inner{\bm{Q}}{\bm{X}}\\
		s.t.&\inner{\bm{A}_{i}}{\bm{X}}= 0, \quad \forall i \in \indset{N_c}\\
		& \inner{\bm{A}_0}{\bm{X}} = 1,\\
		& \bm{X} \succeq \bm{0},\\
		& \mbox{rank}(\bm{X})=1,
	\end{array}
\end{equation}
where the last two constraints implicitly enforce the fact that $\bm{X} = \bm{x}\bm{x}^T$. In Problem~\eqref{opt:R1SDP}, we see that all the non-convexity of the problem has been relegated to a single non-convex constraint on the rank. It follows that we can find a convex relaxation of Problem~\eqref{opt:R1SDP} by removing the rank constraint: 
\begin{equation}\label{opt:SDP}
	\begin{array}{rl}
		\min\limits_{\bm{Z}} & \inner{\bm{Q}}{\bm{Z}}\\
		s.t.&\inner{\bm{A}_i}{\bm{Z}}= 0, \quad \forall i = \left[N_c\right]\\
		& \inner{\bm{A}_0}{\bm{Z}}= 1, \\
		& \bm{Z} \succeq \bm{0}.
	\end{array}
\end{equation}
It is well known that if the optimal solution $ \bm{Z}^* $ satisfies $ \rank{\bm{Z}^*}=1 $ then the convex relaxation is \textit{tight} to the original QCQP. Accordingly, this condition implies that the solution can be factorized as $ \bm{Z}^* = \bm{z}^* \bm{z}^{*T} $, where $ \bm{z}^* $ is the \textit{globally optimal} solution of Problem \eqref{opt:QCQP}.

We note that Problems \eqref{opt:SDP} and \eqref{opt:R1SDP} have the same Lagrangian dual problem, which plays a central role in certifying global optima and can be expressed as follows:
\begin{equation}\label{opt:Dual}
	\begin{array}{rl}
		d^*=\max\limits_{\bm{H},\bm{\lambda},\rho} &-\rho \\ 
		s.t. & \bm{H}(\bm{\lambda},\rho) = \bm{Q} + \rho\bm{A}_0 +\sum\limits_{i=1}^{N_c} \lambda_i \bm{A}_i, \\
		& \bm{H}(\bm{\lambda},\rho) \succeq \bm{0},
	\end{array}
\end{equation}
where $ \rho $ and $ \bm{\lambda} = \left[\lambda_1~\cdots~\lambda_{N_c}\right]^T $ are the Lagrange multipliers corresponding to the single homogenizing and the $N_c$ quadratic constraints, respectively. 

Since Problem \eqref{opt:SDP} is convex, it has the same optimal cost as its dual, Problem \eqref{opt:Dual}.\footnote{This is a consequence of the fact that strong duality holds for convex problems under sufficient constraint qualifications\cite{boydConvexOptimization2004}. In our case, we ensure that the SDPs that we solve satisfy the Linear Independent Constraint Qualification.} When the relaxation is tight, the optimal cost of the original QCQP, Problem \eqref{opt:QCQP}, will also be equal to the dual problem cost, a condition known as \emph{strong duality}.

\rev{The complementary slackness conditions of \eqref{opt:SDP} and its dual, \eqref{opt:Dual}, provide the following relationship between the rank of the primal and dual solution matrices,
\begin{equation}
	\mbox{rank}(\bm{Z}^*) \leq n-\mbox{rank}(\bm{H}^*)=\mbox{corank}(\bm{H}^*),
\end{equation}
where $\bm{H}^*$ represents the dual matrix at the optimal Lagrange multipliers $(\bm{\lambda}^*,\rho^*)$.\footnote{We replace $\bm{H}(\bm{\lambda},\rho)$ with $\bm{H}$ when it is clear from the context.} When \eqref{opt:SDP} satisfies the so-called \emph{strict complementarity} property, this inequality is binding (i.e., an equality). In the sequel, we will make use of this relationship to study the tightness of the relaxation by analyzing the dual matrix, $\bm{H}^*$, which is more directly connected to the problem data.}

In practice, there are typically two ways that globally optimal solutions can be obtained for a QCQP with a tight semidefinite relaxation:
\begin{enumerate}
	\item We can solve Problem \eqref{opt:SDP} or Problem \eqref{opt:Dual} \footnote{The dual problem is solved in~\cite{brialesCartanSyncFastGlobal2017} and is then used to recover the primal solution.} directly and extract the globally optimal solution.
	\item Given a candidate solution, $ \hat{\bm{z}} $, found via fast optimization methods, we can certify its global optimality using the dual or \emph{certificate} matrix, $\bm{H}$.
\end{enumerate}

In robotics, there is a strong preference for the latter method, since it is often much more computationally efficient than the former. State-of-the-art methods leverage low-rank SDP techniques such as that of \emph{Burer and Monteiro} (BM)~\cite{burerNonlinearProgrammingAlgorithm2003a} and the \emph{Riemannian Staircase}~\cite{rosenSESyncCertifiablyCorrect2019, boumalNonconvexBurerMonteiro2016} to solve larger problem instances in real time.

However, these methods are far less performant when redundant constraints are introduced to tighten the semidefinite relaxation.\footnote{For example, the BM approach requires a verification that the solution is a second-order critical point~\cite{boumalDeterministicGuaranteesBurerMonteiro2020a}. When redundant constraints are used, this verification step necessarily involves solving a different (non-low-rank) SDP, since they are not uniquely determined.} As such, a key goal of this paper is to establish whether SDP relaxations of problems with matrix weights require redundant constraints and, if so, whether it is nevertheless possible to find global solutions efficiently enough for robotics applications.

\section{Problem Formulations}\label{sec:Formulations}

In this section, we provide QCQP formulations for the matrix-weighted localization and SLAM problems that are considered herein. We demonstrate how the costs can be formulated in the standard quadratic form of Problem \eqref{opt:QCQP}. In the interest of brevity, we do not explicitly show the conversion of the constraints into the standard form.

\subsection{Matrix-Weighted Localization}\label{sec:Localization}

In this section, we consider localization problems with matrix weights. That is, we consider the problem of determining the estimate of a sequence of poses given matrix-weighted measurements of a set of \emph{known} landmarks. We assume that point correspondences are known and correct. 

The maximum-likelihood estimate of the poses is given by the solution to the following least-squares optimization problem:
\begin{equation}
	\label{opt:Localize}
	\begin{array}{rl}
		\min\limits_{\bm{C}_i,\bm{t}_i, \forall  i \in \VertSetP} &\sum\limits_{(i,k)\in\EdgeSet_m} \bm{e}_{ik}^T \bm{W}_{ik} \bm{e}_{ik} \\
		\mbox{s.t.} &  \bm{C}_i^T \bm{C}_i = \bm{I}, ~\forall i = \VertSetP,\\
		& (\bm{c}_i^{1})^\times\bm{c}_i^{2} - \bm{c}_i^{3} = \bm{0},~\forall i = \VertSetP, \\
		\mbox{where} & \bm{e}_{ik} = \tilde{\bm{m}}_i^{ki} - \bm{C}_i\bm{m}_{0}^{k0} + \bm{t}_i ,
	\end{array}
\end{equation}
where the variable definitions are consistent with those provided in Section~\ref{sec:MeasModels} and $\bm{c}_i^j$ denotes the $j^{th}$ column of matrix $\bm{C}_i$. It has been shown that the two constraints on the $\bm{C}_i$ in \eqref{opt:Localize} are sufficient to ensure that $\bm{C}_i\in\mbox{SO}(3)$ and can be expressed in the standard quadratic form of \eqref{opt:QCQP}~\cite{tronInclusionDeterminantConstraints}. Therefore, this problem is a QCQP since $\bm{e}_{ik}$ is \emph{linear} in $\bm{C}_i$ and $\bm{t}_i$, and the constraints are quadratic. 

It is important to note that this problem is separable in each set of pose variables, since there are no constraints or cost elements linking any two poses. Therefore, the problem may be divided into $N_p$ subproblems, each being equivalent to the matrix-weighted version of Wahba's problem~\cite{chengTotalLeastSquaresEstimate2019, wahbaLeastSquaresEstimate1965}.\footnote{In the remainder of this paper, we refer to single-pose instances of Problem~\ref{opt:Localize} as Wahba's problem, though it is also known as registration or single-pose localization.} In the scalar-weighted case ($ \bm{W}_{ik} = w_{ik} \bm{I} $), there is a closed-form, global solution to Wahba's problem~\cite{hornClosedFormSolutionAbsolute1988}. However, for the general matrix-weighted case, a key simplification of the cost function is no longer possible and solutions must be found iteratively with no guarantee of global optimality~\cite{barfootPoseEstimationUsing2011,barfootStateEstimationRobotics2017, crassidisSurveyNonlinearAttitude2007}. 

We collect the relevant optimization variables into a single vector, $ \bm{x}_i^T = \begin{bmatrix} \bm{c}_i^T &  \bm{t}_i^T & w \end{bmatrix} $, where $ \bm{c}_i=\vect{\bm{C}_i} $, which allows us to re-express the cost element as
\begin{equation}
	J_{ik}= \bm{x}_i^T\bm{Q}_{ik}\bm{x}_i,
\end{equation}
where the symmetric cost matrix, $\bm{Q}_{ik}$, is given in Appendix~\ref{App:LocCost}. We have also introduced a so-called \textit{homogenizing variable}, $ w $, which is subject to the \textit{homogenizing constraint}, $ w^2 = 1 $ and facilitates the reformulation of Problem \eqref{opt:Localize} into the standard form of Problem \eqref{opt:QCQP}~\cite{wiseCertifiablyOptimalMonocular2020, cifuentesLocalStabilitySemidefinite2022}.

The full cost of Problem \eqref{opt:Localize} can be constructed by permuting and summing the matrices, $\bm{Q}_{ik}$, according to edges in $\EdgeSet_m$ and a given pose variable ordering:
\begin{equation}\label{eqn:locVar}
	\bm{z}^T = \begin{bmatrix}
		\bm{c}_{1}^T & \bm{t}_1^T & \cdots &  \bm{c}_{N_p}^T & \bm{t}_{N_p}^{T} & w
	\end{bmatrix}.
\end{equation}

In practice, it is more efficient to leverage the separability of the problem and solve for each pose via separate instances of Wahba's problem. Since each separate problem is small in dimension (13-by-13), its SDP relaxation can be solved quickly using modern interior-point solvers (see Section~\ref{sec:OutdoorLoc} for runtimes with this approach).

The single-pose version of Problem \eqref{opt:Localize} admits a convex semidefinite relaxation that was also explored in~\cite{brialesConvexGlobal3D2017} and~\cite{olssonSolvingQuadraticallyConstrained2008} with the motivation of representing different geometric primitive measurements (lines and planes), rather than anisotropic noise. It was found that the addition of a few key redundant constraints made the relaxation tight even when noise levels were high. We corroborate these results for anisotropic noise with an extensive analysis in Section~\ref{sec:Simulations}.

We could potientially introduce the \textit{relative-pose measurements} described in Section~\ref{sec:RelPoseMeas} to Problem \eqref{opt:Localize} by adding the associated cost factors to the objective. The cost function given in \eqref{eqn:rel_pose_cost} is a quadratic function in the optimization variables and can be expressed in the standard homogeneous QCQP form, as shown in Appendix~\ref{App:LocCost}. However, the addition of such factors to Problem \eqref{opt:Localize} destroys the separability property, meaning that directly solving the SDP relaxation becomes much slower for reasonably sized problems (e.g., $N_p \geq 20$).

In Appendix~\ref{App:SDPStability}, we prove that the SDP relaxation of Problem \eqref{opt:Localize} is always tight when measurement noise levels are low and the weighting matrices are non-degenerate.

\subsection{Matrix-Weighted Landmark-based SLAM}\label{sec:SLAM}

In this section, we explore the effect of matrix weighting when the landmark locations, $\bm{m}_{0}^{k0}$, are not known \emph{a priori}. The resulting problem, known as landmark-based SLAM, is given by,

\begin{equation}
	\label{opt:landmark_SLAM}
	\hspace*{-10pt}\begin{array}{rl}
		\min\limits_{\substack{\bm{C}_i,\bm{t}_i, \bm{m}_{0}^{k0}\\ \forall i\in \VertSetP,\\ \forall k\in\VertSetM } } &\sum\limits_{(i,k)\in\EdgeSet_m}  \bm{e}_{ik}^T \bm{W}_{ik} \bm{e}_{ik} +\sum\limits_{(i,j)\in\EdgeSet_p}  \bm{e}_{ij}^T \bm{W}_{ij} \bm{e}_{ij}\\
		\mbox{where} & \bm{e}_{ik} = \tilde{\bm{m}}_i^{ki} - \bm{C}_{i0}\bm{m}_{0}^{k0} + \bm{t}_i^{i0} , \\
		&\bm{e}_{ij} = \vect{\tilde{\bm{T}}_{ij}\bm{T}_{j0} - \bm{T}_{i0}}, \\
		\mbox{s.t.} &\bm{C}_i^T \bm{C}_i = \bm{I},~\forall i \in \VertSetP.\\
		& (\bm{c}_i^{1})^\times\bm{c}_i^{2} - \bm{c}_i^{3} = \bm{0},~\forall i = \VertSetP,.
	\end{array}
\end{equation}
Note that the cost of \eqref{opt:landmark_SLAM} includes both landmark measurements and relative-pose measurements.
In the \emph{scalar-weighted} context, the convex relaxation of an equivalent problem has already been studied and was generally found to be tight for noise levels well above those found in practical robotics scenarios~\cite{holmesEfficientGlobalOptimality2023}. 

On the other hand, when matrix weights are used (i.e., the noise distribution is anisotropic), the landmark-based error term \eqref{eqn:lm_error_term} becomes \emph{quadratic} in the optimization variables (due to the $\bm{C}_i\bm{m}_{0}^{k0}$ terms) and the cost function of Problem \eqref{opt:landmark_SLAM} becomes \emph{quartic}. In the scalar-weighted case ($\bm{W}_{ik} = w_{ik}\bm{I}$), this issue is obviated by premultiplying the error terms by the inverse pose rotations to regain \emph{linearity} of the error terms. However, preforming this operation in the matrix-weighted case will change the cost function, since $\bm{C}_i^T\bm{W}_{ik}\bm{C}_i\neq \bm{W}_{ik}$ when $\bm{W}_{ik}$ is an arbitrary positive definite matrix.\footnote{This reflects the fact that the measurement noise \emph{has an orientation} and depends on the (unknown) observation frame. This would not be an issue if all measurements were defined in a common frame, but, in practice, measurements are typically taken in the robot's frame of reference.}

In order to cast Problem \eqref{opt:landmark_SLAM} as a QCQP, we follow a similar strategy to~\cite{dumbgenSafeSmoothCertified2023} and~\cite{brialesCertifiablyGloballyOptimal2018} by introducing \textit{substitution variables}, $ \bm{m}_i^{ki} $, corresponding to each available measurement, $ \tilde{\bm{m}}_i^{ki} $. These new variables must satisfy the following (quadratic) constraints:
\begin{equation}
	\bm{m}_i^{ki} = \bm{C}_i\bm{m}_{0}^{k0} - \bm{t}_i, \quad \forall (i,k)\in\EdgeSet_m.
\end{equation}

When the landmarks become unknown, a gauge freedom is also introduced into the problem. This has an important implication for the SDP solution; the well-known SLAM gauge freedom results in solution symmetries that cause the rank of the SDP solution to be higher than one, even when it is numerically tight~\cite{brialesCertifiablyGloballyOptimal2018}. To fix this freedom, we assume that a \emph{prior} factor term is included in the pose-graph terms for at least one pose.\footnote{As with SLAM problems in general, the pose variable in the prior, $\tilde{\bm{T}}_{j0}$, is typically used to `lock' the solution to a known pose using some exteroceptive measurement, such as GPS. However, for our purposes it can be set arbitrarily without loss of generality.}

The optimization can now be written as
\begin{equation}
	\hspace*{-5pt}
	\label{opt:SLAM}
	\begin{array}{rl} 
		\min \limits_{ \substack{\bm{C}_i,\bm{t}_i, \\ \bm{m}_{0}^{k0},\bm{m}_i^{ki} \\ \forall i\in \VertSetP,~ \forall k\in\VertSetM } } & \hspace*{-20pt}\sum\limits_{(i,k)\in\EdgeSet_m} \bm{e}_{ik}^T \bm{W}_{ik} \bm{e}_{ik} + \hspace*{-7pt}\sum\limits_{(i,j)\in\EdgeSet_p}  \bm{e}_{ij}^T \bm{W}_{ij} \bm{e}_{ij}\hfill \\
		\mbox{where} & \bm{e}_{ik} = \tilde{\bm{m}}_i^{ki} - \bm{m}_i^{ki},\\
		&\bm{e}_{ij} = \vect{\tilde{\bm{T}}_{ij}\bm{T}_j - \bm{T}_i},\\
		\mbox{s.t.}&\bm{C}_i^T \bm{C}_i = \bm{I},~\forall i \in \VertSetP.\\
		& (\bm{c}_i^{1})^\times\bm{c}_i^{2} - \bm{c}_i^{3} = \bm{0},~\forall i = \VertSetP,\\
		& \bm{m}_i^{ki} = \bm{C}_i\bm{m}_{0}^{k0} - \bm{t}_i, \quad \forall (i,k)\in\EdgeSet_m.
	\end{array}
\end{equation}
The pose-landmark cost elements can now be written in the standard form:
\begin{equation*}
	J_{ik}=\begin{bmatrix}
		\bm{m}_i^{ki}\\w
	\end{bmatrix}^T \begin{bmatrix}
		\bm{W}_{ik} & -\bm{W}_{ik}\tilde{\bm{m}}_i^{ki}\\
		-\tilde{\bm{m}}_i^{ki^T}\bm{W}_{ik} & \tilde{\bm{m}}_i^{ki^T}\bm{W}_{ik}\tilde{\bm{m}}_i^{ki}
	\end{bmatrix}\begin{bmatrix}
	\bm{m}_i^{ki}\\w
	\end{bmatrix}
\end{equation*}
The cost elements can be then permuted and summed according to our variable ordering,
\begin{multline}
	\bm{z}^T = \left[\begin{matrix}\vect{\bm{C}_{10}}^T &\bm{t}_1^{10^T} & \cdots &\vect{\bm{C}_{N_p0}}^T&\bm{t}_{N_p}^{N_p0^T} \end{matrix}\right. \\
		\left.\begin{matrix} \bm{m}_{0}^{k0^T} &\cdots& \bm{m}_i^{ki^T} & \cdots & w \end{matrix}\right],
\end{multline}
and we can apply the semidefinite relaxation described in~\ref{sec:Relaxation}. However, we have observed that, contrary to the scalar case, this SDP relaxation is \emph{not tight} even for low levels of noise. A similar situation occurred in~\cite{brialesCertifiablyGloballyOptimal2018} and~\cite{yangTEASERFastCertifiable2021} when substitution variables were introduced (though not when introduced in~\cite{dumbgenSafeSmoothCertified2023}).

\subsection{Tightening the Relaxations}\label{sec:Tightening}

One of the key contributions of this paper is a concise set of redundant constraints for the problems presented above that is capable of tightening their respective relaxations. In the next sections, we empirically show that these constraints are capable of restoring relaxation tightness that is otherwise destroyed by the introduction of anisotropic noise.

One approach that can be used to tighten this relaxation is to apply the Lasserre-moment hierarchy~\cite{henrionMomentSOSHierarchyLectures2021, lasserreGlobalOptimizationPolynomials2001}. However, this method often introduces a \emph{prohibitive} number of additional  variables and constraints to the problem~\cite{yangCertifiablyOptimalOutlierRobust2023}. 

In constrast, we have opted to \emph{discover} a smaller set\footnote{Smaller in the sense that we do not include \emph{all} possible constraints or further ascend Lasserre's hierarchy.} of redundant constraints that is sufficient to render each of the problems above tight for reasonable levels of noise.  To do this, we leverage a \emph{constraint-learning} concept from our concurrent paper, which, \emph{for a given problem instance}, uses samples that are drawn from the feasible set to numerically find all possible constraints via a nullspace argument~\cite{dumbgenGloballyOptimalState2023a}. \rev{Since it involves finding the nullspace of a large data matrix, this constraint-learning process can be computationally intensive and it is not feasible to repeat it for each new problem instance.}

\rev{Therefore, we}applied the constraint-learning method to small instances of the problems defined above \rev{and} interpreted the \rev{resulting} numerical, learned constraints. \rev{We then} cast them as equations that are commensurate with the properties of $\mbox{SO}(3)$ and could be extended to problems of any size (any number of landmarks and poses). 

\rev{For example, for Problem~\eqref{opt:SLAM}, we randomly initialized a problem with three poses and seven landmarks and drew samples from the feasible set. Using these samples, the methodology in~\cite{dumbgenGloballyOptimalState2023a} allowed us to numerically find a set of viable constraint matrices for this problem. An example of one such constraint matrix is shown in Figure~\ref{fig:constraint_ex}. The constraint matrix corresponds to the equation $\bm{z}^T \bm{A} \bm{z} = 0$, where $\bm{A}$ is the matrix and $\bm{z}$ is a subset of variables affected by this constraint. In this case, 
\begin{equation*}
	\bm{z}^T = \begin{bmatrix}
		w & \vect{\bm{C}_1}^T & \bm{m}^{20^T}_0 & \bm{m}^{40^T}_0 & \bm{m}^{21^T}_1 & \bm{m}^{41^T}_1
	\end{bmatrix}.
\end{equation*}
The constraint can be interpreted as the following polynomial equation:
\begin{align*}
	&-  w \left[\bm{m}_1^{21}\right]_3 +   w \left[\bm{m}_1^{41}\right]_3 +   \left[  \vect{\bm{C}_1}\right]_3 \left[\bm{m}_0^{20}\right]_1 \\
	&+ \left[  \vect{\bm{C}_1}\right]_6 \left[\bm{m}_0^{20}\right]_2 +   \left[  \vect{\bm{C}_1}\right]_9 \left[\bm{m}_0^{20}\right]_3 \\
	& -  \left[  \vect{\bm{C}_1}\right]_3 \left[\bm{m}_0^{40}\right]_1  -  \left[  \vect{\bm{C}_1}\right]_6 \left[\bm{m}_0^{40}\right]_2 \\
	&- \left[  \vect{\bm{C}_1}\right]_9 \left[\bm{m}_0^{40}\right]_3=0,
\end{align*}
where $\left[\bm{v}\right]_i$ denotes the $i^{th}$ element of vector $\bm{v}$. Further study of this equation in light of the properties of $\mbox{SO}(3)$ reveals that it can be generalized as the third vector component of the following constraint equation:
\begin{equation*}
	\bm{C}_i\left(\bm{m}_{0}^{l0}-\bm{m}_{0}^{k0}\right) - \left(\bm{m}_i^{li}-\bm{m}_i^{ki}\right) = \bm{0},
\end{equation*}
which describes distances between landmarks in the world and $i^{th}$ frames. The \emph{analytical} version of this constraint can be applied to other problem instances (i.e., with different numbers of landmarks and poses). Since it was found to be effective in tightening other instances of Problem~\eqref{opt:SLAM}, we added it to our list of effective constraints,~\eqref{eqn:slam_constraints}. Proceeding in this manner, we selected a set of constraints that were empirically found to lead to tightness across various problem instances.}

\begin{figure}[!t]
	\centering
	\includegraphics[width=\columnwidth]{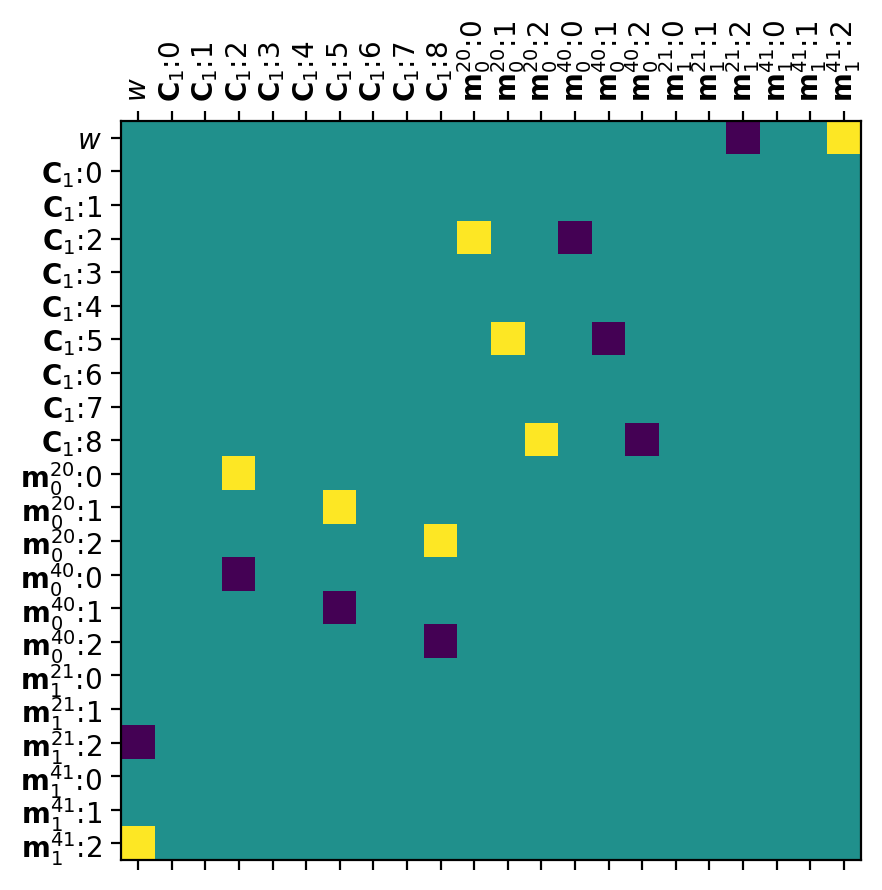}
	\caption{\rev{Redundant constraint matrix, $\bm{A}$, found numerically via the method in~\cite{dumbgenGloballyOptimalState2023a}, with colours indicating the relative values of the matrix (teal: 0, yellow: 1, purple: -1). The constraint matrix corresponds to the equation $\bm{x}^T \bm{A} \bm{x} = 0$, where $\bm{x}$ is a subset of variables affected by this constraint.}}
	\label{fig:constraint_ex}
\end{figure}

The concept of SDP stability~\cite{cifuentesLocalStabilitySemidefinite2022} provides \emph{some} guarantees that similar problems to the ones that we study will also be tight \rev{with these constraints}, though we cannot \emph{a priori} claim that \emph{any} instance of Problem \eqref{opt:SLAM} will be tight. 

We found that the following set of redundant constraints could be used to tighten Problem \eqref{opt:Localize}:
\begin{subequations}\label{eqn:so3_constraints}
	\begin{flalign}
		\bm{C}_i \bm{C}_i^T - \bm{I} = \bm{0},
	\end{flalign}
	\vspace*{-20pt}
	\begin{flalign}
		(\bm{C}_i^{j})^\times\bm{C}_i^{k} - \bm{C}_i^{l} = \bm{0}, ~\forall (j,k,l)\in \left\{(2,3,1), (3,1,2)\right\},
	\end{flalign}
	\vspace*{-20pt}
	\begin{flalign}
		\bm{c}_i^{j^T}\bm{c}_i^{j} - \bm{r}_{i}^{k^T}\bm{r}_{i}^{k} = \bm{0}, ~\forall j,k\in \left\{1,2,3\right\},
	\end{flalign}
\end{subequations}
where $\bm{c}_i^{j}$ and $\bm{r}_{i}^{k}$ represent the $j^{th}$ column and $k^{th}$ row of $\bm{C}_i$, respectively, and $\mbox{cyclic}()$ denotes the cyclic group. Note that the first two redundant constraints were also necessary in~\cite{brialesConvexGlobal3D2017} and~\cite{wiseCertifiablyOptimalMonocular2020}. 

Moreover, our SLAM problem (Problem \eqref{opt:SLAM}) can be tightened by using \eqref{eqn:so3_constraints} in conjunction with the following constraints:
\begin{subequations}\label{eqn:slam_constraints}
	\begin{flalign}
		\bm{m}_{0}^{l0^T} \bm{m}_{0}^{k0} - (\bm{m}_i^{li}-\bm{t}_i)^T (\bm{m}_i^{ki}-\bm{t}_i) = 0,% landmark dot products
	\end{flalign}
	\vspace*{-20pt}
	\begin{flalign}\label{eqn:constr_ex}
		\bm{C}_i\left(\bm{m}_{0}^{l0}-\bm{m}_{0}^{k0}\right) - \left(\bm{m}_i^{li}-\bm{m}_i^{ki}\right) = \bm{0}, % landmark to landmark vector in pose frame
	\end{flalign}
	\vspace*{-20pt}
	\begin{flalign}
	 	\left(\bm{m}_{0}^{l0}-\bm{m}_{0}^{k0}\right) - \bm{C}_i^T\left(\bm{m}_i^{li}-\bm{m}_i^{ki}\right) = \bm{0}, % landmark to landmark vector in world frame
	\end{flalign}
	\vspace*{-20pt}
	\begin{flalign}
	 	\left\Vert\bm{m}_i^{li}-\bm{m}_i^{ki}\right\Vert^2 - \left\Vert\bm{m}_j^{lj}-\bm{m}_j^{kj}\right\Vert^2 = 0, % Landmark differences in different frames
	\end{flalign}
	\vspace*{-20pt}
	\begin{flalign}
	 	\left(\bm{m}_{0}^{l0}-\bm{m}_{0}^{k0}\right)^{\times} \bm{C}_i - \bm{C}_i \left(\bm{m}_{0}^{l0}-\bm{m}_{0}^{k0}\right)^{\times} =\bm{0}, % Adjoint Equation
	\end{flalign}
	\vspace*{-20pt}
	\begin{flalign}
		\bm{t}_i - \left(\frac{1}{N_i}\sum_{i=1}^{N_i}\bm{C}_i \bm{m}^{k0}_0 - \bm{m}^{ki}_i\right) = \bm{0}, % average pose
	\end{flalign}
	\vspace*{-10pt}
	\begin{flalign*}
		\hspace*{50pt}\forall i,j,l,k~ \mbox{s.t.}~ \left\{(i,l),(i,k),(j,l),(j,k)\right\}\subset\EdgeSet_m.\nonumber
	\end{flalign*}
\end{subequations}

Despite the fact that these constraints successfully tighten our problems for practical noise levels, the presence of redundant constraints prohibits the use of fast certification methods as mentioned above, meaning that interior-point methods must be used to certify or solve the relaxation. Therefore, it is very important to characterize the noise regime for which the convex relaxation of Problem \eqref{opt:landmark_SLAM} is still tight, which is the subject of Section~\ref{sec:NoiseAnalysis}.

Due to the large number of variables and constraints that must be introduced, the problem sizes that can be solved using this method are still small. This may be mitigated by exploiting the sparsity of the problem, but this remains as future work. One notable exception is the localization problem without relative-pose measurements, which is still reasonably tractable due to the separability of the problem.

\section{Estimation Uncertainty and the Dual Certificate}\label{sec:Uncertainty}

In this section, we present a set of theoretical results that allows us to extend intuitions about uncertainty in state estimation to tightness of SDP relaxations. We establish a connection between the dual solution of the SDP relaxation (specifically the certificate matrix, $\bm{H}$) and the Fisher Information Matrix. 

The following Lemma relates the \rev{certificate matrix to} the Laplace approximation of an equivalent (local) unconstrained problem of the form \eqref{opt:UnconstrainedMAP}. \rev{Throughout this section, we will assume that we have access to the global minimum and explore properties of its associated certificate matrix.}

\begin{lemma}\label{lem:FisherInfo}
	Suppose a given MAP estimation problem can be equivalently formulated as either a standard-form, homogenized QCQP \eqref{opt:QCQP} or as an unconstrained optimization as in \eqref{opt:UnconstrainedMAP}. Let $\bm{z}^*$ and $\bm{x}^*$ be the (global) optima of these formulations, respectively. Given a neighborhood, $\mathcal{U} \subseteq \mathbb{R}^p$, containing $\bm{x}^*$, let $\bm{\ell}: \mathbb{R}^p \mapsto \mathbb{R}^n$ be a \emph{smooth, injective} map on $\mathcal{U}$, such that its image is in the feasible set of \eqref{opt:QCQP}.\footnote{The existence of the mapping, $\bm{\ell}$, is not restrictive in our context. In fact, since the feasible set of the problems in this paper are smooth manifolds, this mapping can be interpreted as the \emph{inverse coordinate chart} from differential geometry.} Moreover, let the respective objective functions be equal under the map in the neighborhood $\mathcal{U}$. That is,
	\begin{equation}
		-\log\left(p(\bm{x} \vert \bm{\mathcal{D}})\right) = \bm{z}^T \bm{Q} \bm{z},
	\end{equation}
	for all $(\bm{x}, \bm{z})$ such that $\bm{z}=\bm{\ell}(\bm{x})$ and $\bm{x} \in \mathcal{U}$. Then the FIM of $p(\bm{x} \vert \bm{\mathcal{D}})$ is given by
	\begin{equation}\label{eqn:FIM_H}
		\bm{\Sigma}^{-1}= \bm{L}^T \bm{H}\bm{L},
	\end{equation}
	where $\bm{H}$ is the certificate matrix at the solution and $\bm{L}$ is the Jacobian of $\bm{\ell}(\bm{x})$ at the solution, $\bm{L} = \frac{d}{d\bm{x}}\bm{\ell}(\bm{x})\vert_{\bm{x}=\bm{x}^*}$. 
\end{lemma}
We defer the proof of this Lemma to Appendix~\ref{App:lemma1Proof}. Intuitively, the Lemma uses the fact that the local curvature of the two problems are equal on the feasible set to establish a connection between the FIM and the certificate matrix. 

One implication of this Lemma is that the certificate matrix can be interpreted as an information matrix in the higher-dimensional, SDP space. It also provides a method to extract posterior covariance matrices from a given SDP solution. 

\rev{\subsection{Application to Localization}\label{sec:fim-wahba-ex}

To make the Lemma in the preceding section more concrete,  we show its application to Problem~\eqref{opt:Localize}. The equivalent unconstrained parameterization of this problem is the Lie-algebra vector space parameterization (see \cite{barfoot2011state,dellaertShonanRotationAveraging2020} for more details on Lie groups in robotics).

For the $i^{th}$ pose variable in Problem~\eqref{opt:Localize}, let $\bm{x}_i= \begin{bmatrix} \bm{x}_{i,r}^T & \bm{x}_{i,t}^T\end{bmatrix}^T \in \mathbb{R}^3\times\mathbb{R}^3$ be the associated variable for the unconstrained problem. We have the following relationships:
\begin{equation}
	\bm{C}_i = \exp(\bm{x}_{i,r}^\wedge)\bar{\bm{C}}_i, \quad \bm{t}_i = \bm{x}_{i,t} + \bar{\bm{t}}_i,
\end{equation} 
where $\bar{\bm{C}}_i$ and $\bar{\bm{t}}_i$ are the optimal values of $\bm{C}_i$ and $\bm{t}_i$, respectively, for a given problem instance, $\exp$ is the matrix exponential and $^\wedge$ is the skew-symmetric operator (see~\cite{barfoot2011state} for details). We can rewrite Problem~\eqref{opt:Localize} as the unconstrained problem,
\begin{equation}
	\label{opt:LocalizeUnc}
	\begin{array}{rl}
		\min\limits_{\bm{x}_i \in \mathbb{R}^6} &\sum\limits_{(i,k)\in\EdgeSet_m} \bm{e}_{ik}^T \bm{W}_{ik} \bm{e}_{ik} \\
		\mbox{where} & \bm{e}_{ik} = \tilde{\bm{m}}_i^{ki} - \exp(\bm{x}_{i,r}^\wedge)\bar{\bm{C}}_i\bm{m}_{0}^{k0} + \bm{x}_{i,t} + \bar{\bm{t}}_i,
	\end{array}
\end{equation}
From \eqref{eqn:locVar}, we see that the solution mapping is given by
\begin{gather}
	\bm{z}^T=\begin{bmatrix}\bm{z}^T_1&\cdots& \bm{z}^T_n&1\end{bmatrix} \\
	\bm{z}_i=\bm{\ell}(\bm{x}_i) = \begin{bmatrix}
		\vect{\exp(\bm{x}_{i,r}^\wedge)\bar{\bm{C}}_i}\\ \bm{x}_{i,t} + \bar{\bm{t}}_i
	\end{bmatrix}.
\end{gather}
By construction, this map is bijective with its image in a neighborhood around $(\bar{\bm{C}}_i, \bar{\bm{t}}_i)$ and the cost functions of Problems~\eqref{opt:Localize} and~\eqref{opt:LocalizeUnc} are equal under the map. Finally, the map is smooth by the properties of Lie groups and its Jacobian is given by
\begin{equation}
	\bm{L} = \mbox{diag}(\bm{L}_1,\dots,\bm{L}_n,0), \quad
	\bm{L}_i = \begin{bmatrix}
		(\bm{I}\otimes\bar{\bm{C}}_i) \bar{\bm{G}}_d & \bm{0}\\
		\bm{0} & \bm{I}
	\end{bmatrix},
\end{equation}
where $\bar{\bm{G}}_d$ is a matrix of the vectorized generators of the $\mbox{SO}(3)$ Lie algebra (see equations 7-10 of~\cite{dellaertShonanRotationAveraging2020} for a detailed derivation of the top-left block) and $\mbox{diag}()$ denotes block-diagonal concatenation. In this example, it can be shown that the minimum singular value of $\bm{L}$ is exactly unity and, if applying Lemma~\ref{lem:FisherInfo}, we have that the minimum eigenvalue of the parameterized FIM exactly upper bounds the minimum eigenvalue of the certificate matrix.

We visualize Lemma~\ref{lem:FisherInfo} on a two-pose localization problem with stereo measurements by comparing the numerical covariance matrix (obtained by running 10,000 optimization trials with different sampled noise values) with the theoretical covariance matrix (inverse of FIM of the final sample). The matrices match to numerical precision and can be seen in Figure~\ref{fig:num_cov}. 
\begin{figure}[!ht]
	\centering
	\includegraphics[width=\columnwidth]{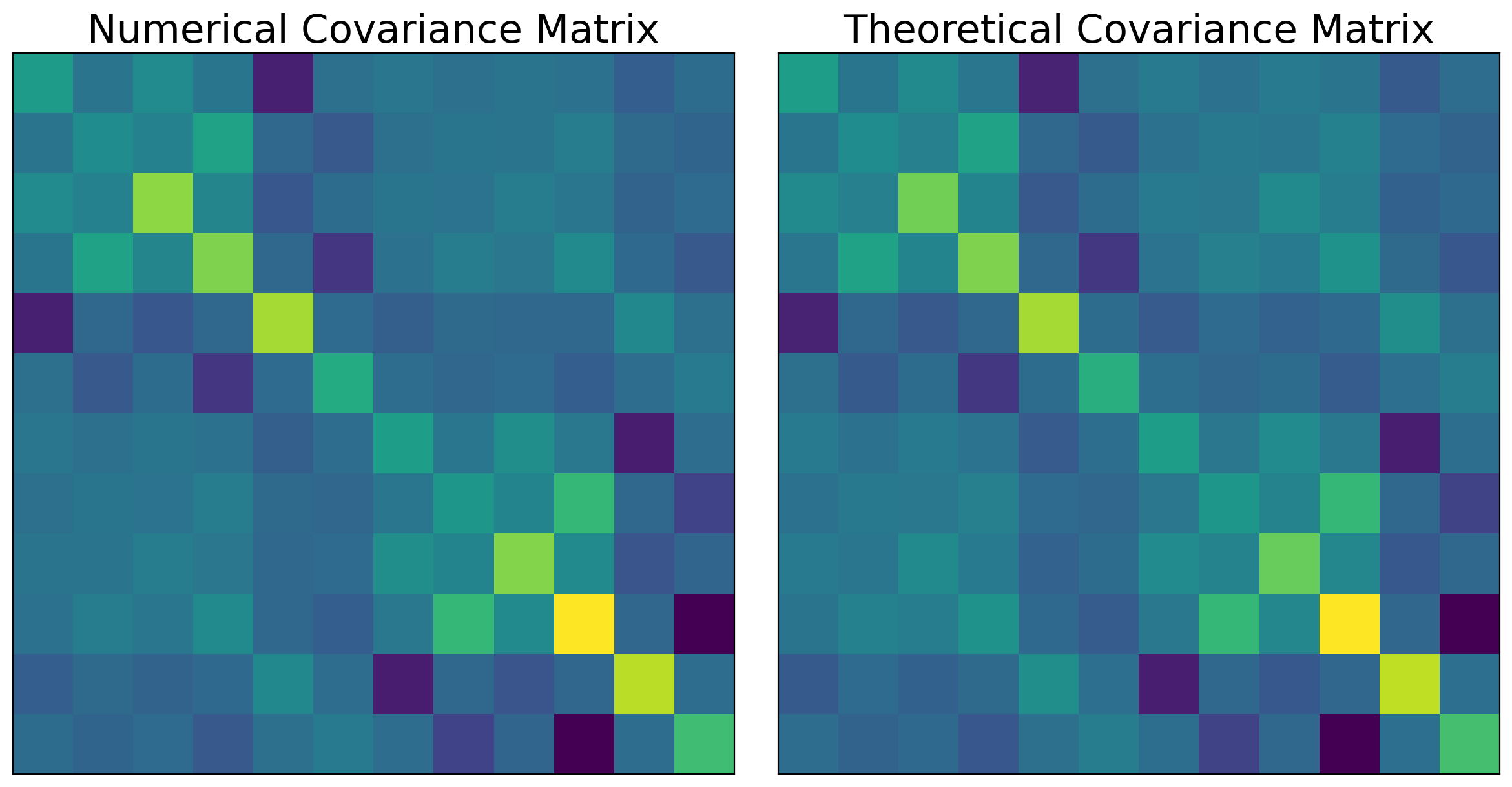}
	\caption{\rev{Comparison of numerical covariance matrix to the theoretical covariance matrix from Lemma~\ref{lem:FisherInfo}, $\bm{\Sigma} = (\bm{L}^T\bm{H}\bm{L})^{-1}$, for a two-pose, stereo localization problem, Problem~\ref{opt:Localize}. Numerical covariance was found by considering the sample covariance of 10,000 (globally optimal) estimates. Note that the covariance matrix is fully coupled because we have included a relative-pose measurement.}}
	\label{fig:num_cov}
\end{figure}
}
\rev{\subsection{Connection to Relaxation Tightness}

In this section, we aim to connect posterior estimation uncertainty to tightness of the SDP relaxation. To this end, we present a Lemma that relates the minimum eigenvalues of the certificate matrix and the FIM.
\begin{proposition}\label{prop:eig_bounds}
	Assume that the setting of Lemma~\ref{lem:FisherInfo} applies. Without loss of generality, we assume that the homogenizing variable is the last element of $\bm{z}$. Then the certificate matrix can be partitioned as,
	\begin{equation*}
		\bm{H}=\begin{bmatrix}
			\bar{\bm{H}} & \bm{h} \\ \bm{h}^T & h
		\end{bmatrix},
	\end{equation*}
	and we have the following relation:
	\begin{equation}\label{eqn:eig_bound}
		\sigma_{\min}(\bar{\bm{H}})  \leq \frac{\sigma_{\min}(\bm{\Sigma}^{-1})}{\rev{s_{\min}(\bm{L})^2}},
	\end{equation}
	\rev{where $\sigma_{\min}(\cdot)$ denotes the minimum eigenvalue and $s_{\min}(\cdot)$ denotes the smallest singular value.}
\end{proposition}
The proof of this lemma is given in Appendix \ref{App:lem2Proof}. It is known that high posterior uncertainty along particular dimensions of state-estimation problems leads to degenerate (or nearly degenerate) FIMs\cite{zhangDegeneracyOptimizationbasedState2016}. If strict complementarity holds for the SDP,\footnote{\rev{Strict complementarity is a property of SDPs that has been shown to hold \emph{generically}~\cite{alizadehComplementarityNondegeneracySemidefinite1997}. However, this property has been shown to fail in highly structured SDPs, such as relaxations of some combinatorial problems\cite{cifuentesLocalStabilitySemidefinite2022,silvaStrictComplementarityMaxCut2018}. Nevertheless, we have empirically observed that this property holds for the problems we consider here.}} then as the FIM approaches degeneracy (i.e., its minimum eigenvalue approaches zero) it can be shown (c.f. Appendix~\ref{app:fim-degeneracy}) that either the certificate matrix has a negative eigenvalue at the global solution or the primal SDP solution is not rank-one. In either case, the practical implication in such a scenario is that tightness of the SDP relaxation is lost (regardless of the presence of redundant constraints). 

Our investigations in the next section imply that posterior uncertainty plays a role even when the FIM is \emph{approximately} degenerate (i.e., scenarios with high or concentrated noise), but investigation of the theoretical mechanism explaining this phenomenon is still ongoing. For example, Figure \ref{fig:cert_eig_study} shows that the minimum eigenvalue of the FIM generally upper bounds the second eigenvalue of the certificate matrix across parameters for Wahba's problem.} 

We will see in next section that the tightness of the semidefinite relaxation degrades as the uncertainty of the optimal state estimate increases (i.e., as the FIM becomes more degenerate). In particular, the introduction of anisotropic noise can cause accentuated uncertainty in a particular direction and erode the tightness of the semidefinite relaxation.

\section{Simulated Experiments}\label{sec:Simulations}

It is known that the tightness of SDP relaxations of least-squares perception problems depends on the level of noise in the measurements~\cite{brialesConvexGlobal3D2017, rosenSESyncCertifiablyCorrect2019, cifuentesLocalStabilitySemidefinite2022}. For the case of rotational averaging, tightness of the relaxation has been linked to the magnitude of \textit{residual uncertainty} of pose estimates~\cite{erikssonRotationAveragingStrong2018}.

In this section, we empirically explore the effect that introducing anisotropic noise and matrix weights have on localization and SLAM. Our study is strongly motivated by stereo-camera noise models, but not limited thereto. 
All of the results in this and the next section were generated using MOSEK's interior-point, SDP solver~\cite{mosek}.

It is common in the literature report average rank when assessing tightness of the SDP relaxation, but we find that the eigenvalue  ratio (ER) of the optimal solution  -- that is, the ratio between the first and second eigenvalues -- is a more informative metric for tightness. Generally, $\mbox{ER}\geq 10^6$ is an appropriate indicator that an SDP relaxation is rank-one and therefore tight. In the subsequent analysis, we use this metric as the main criterion for tightness. 

As mentioned above, when measurements are based on a stereo-camera model, the uncertainty ellipsoids in Euclidean space become elongated. This elongation occurs along rays extending from the camera focal point and depends quadratically on the distance to the measured point. To capture this elongation, we define the \textit{anisotropicity} of a measurement as the square root of the conditioning number of the noise covariance matrix (i.e., square root of ratio of maximum to minimum eigenvalues). Figure~\ref{fig:ellipsoid_align}(a) presents a visualization of the shape of an uncertainty ellipsoid as anisotropicity changes.

All of the studies in this section have the same general format. Landmark locations were randomly generated from a uniform distribution within a bounding cube of a given size and distance relative to the pose/camera frame. Unless otherwise specified, the default values for distance and bounding cube length are 3 m and 1 m, respectively. 

We searched for sets of parameters that have tight relaxations through a variety of analyses. For each analysis, the parameter space was sampled using a 30-by-30 grid in logspace across the ranges shown in the figures. For each point in the parameter space, the SDP was solved for 100 random landmark and pose geometries generated using random seeds that were consistent across parameter points. For each point in parameter space, the minimum ER was found across all trials and boundaries were plotted along the parameter values at which the minimum ER dropped below $1\times10^{-6}$.\footnote{Note that in reality, the raw contours are quite noisy. For convenience to the reader, we first smooth the minimum ER values with a median filter, then plot the contours. The smoothing method is reviewed Appendix~\ref{App:Smoothing}.}

\subsection{Anisotropic Noise in Wahba's Problem}\label{sec:NoiseAnalysis}

In this section, we directly control the level and alignment of measurement noise anisotropicity (i.e., uncertainty ellipoid size and shape) and observe their effect on the tightness of Wahba's problem. 

\subsubsection{Aligned Uncertainty Ellipsoids}

We first study the effect of anisotropic noise  on Wahba's problem and consider the simplified case in which all of the error ellipsoids are aligned. The alignment of the ellipsoids causes the uncertainty to be concentrated in a single direction leading to high posterior uncertainty in that direction~\cite{zhangDegeneracyOptimizationbasedState2016}. This case is of interest since our analysis in Section~\ref{sec:Uncertainty} suggests that there is a connection between high posterior uncertainty and loss of tightness. 

Figure~\ref{fig:ellipsoid_align} shows the boundary between tight and non-tight SDP relaxations for this problem as the standard deviation (STD) of the noise, anisotropicity of the noise  and number of landmarks are varied. The problem setup is shown in Figure~\ref{fig:ellipsoid_align}(b), while Figure~\ref{fig:ellipsoid_align}(a) demonstrates the effect of increasing anisotropicity on the uncertainty ellipsoid.

Figure~\ref{fig:ellipsoid_align} (c) and (d) show the parameter-space regions that have tight relaxations without and with (respectively) the redundant constraints in \eqref{eqn:so3_constraints}. For the case without redundant constraints, we see that when anisotropicity is close to one (i.e., nearly isotropic), the noise level that yields tight relaxations is high even for low numbers of landmarks. These results are consistent with previous results for isotropic noise~\cite{holmesEfficientGlobalOptimality2023}. Increasing the number of landmarks generally improves tightness for a given noise level. We also see that for any given level of noise, increasing anisotropicity eventually results in a loss of tightness. 

\rev{We observe} that the redundant constraints effectively make the problem tight across almost all parameters studied. The effect of the redundant constraints as well as the relationship between tightness, the certificate matrix and the FIM are explored in greater depth in Figure~\ref{fig:cert_eig_study}. This figure shows heatmaps for the 50 landmark case of Figure~\ref{fig:ellipsoid_align}, without and with redundant constraints, across three metrics: the solution ER (subplots (a) and (b)), second smallest eigenvalue of the certificate matrix (subplots (c) and (d)), and the minimum eigenvalue of the FIM (\rev{subplots} (e) and (f)). The FIM was calculated using \eqref{eqn:FIM_H} with the mapping $\bm{L}$ as defined in \rev{Section \ref{sec:fim-wahba-ex}}. 

A number of key observations can be drawn. First, we note that the tightness boundary in (a) (magenta line)\footnote{Note that this line is the same as the same boundary as shown in Figure~\ref{fig:ellipsoid_align}(c).} coincides almost exactly with the drop in the certificate matrix eigenvalue in (c), as expected.\footnote{This actually demonstrates that strict complementarity mostly holds for this problem, since a drop in rank of the certificate is exactly complemented by an increase in rank of the solution matrix.} 

Second, by comparing (c), (d) and (f), we see that the introduction of redundant constraints increases the certificate eigenvalue to close to the FIM upper bound (though it does not quite attain it). It is interesting to note that the bound in \rev{Proposition}~\ref{prop:eig_bounds} is quite loose without redundant constraints. 

Finally, we note that the FIM changes very little when redundant constraints are introduced. This is expected, since the FIM for a given problem should not depend on the redundant constraints. We posit that slight differences that we observe between plots (e) and (f) are numerical in nature. 

\begin{figure}[!t]
	\centering
	%\vspace*{-0.05in}
	\includegraphics[width=\columnwidth]{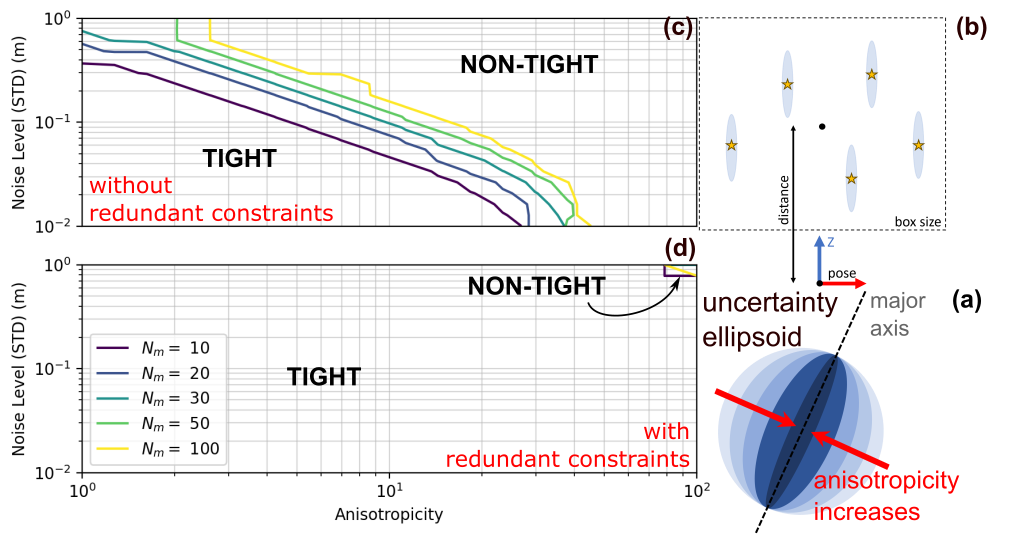}
	%\vspace*{-0.3in}
	\caption{Investigation of the effect of anisotropicity on tightness of semidefinite relaxations for Wahba's problem.  Subplot (a) shows how varying anisotropicity affects the uncertainty ellipsoid. Subplot (b) shows the problem setup, with ellipsoids aligned to the $z$-axis of the pose. Subplot (c) shows the tightness boundary for varying numbers of landmarks. Anisotropicity decreases the tightness boundary while increasing number of observed landmarks increases the boundary.}
	%\vspace*{-0.22in}
	\label{fig:ellipsoid_align}
\end{figure}

\begin{figure}[!t]
	\centering
	\includegraphics[width=\columnwidth]{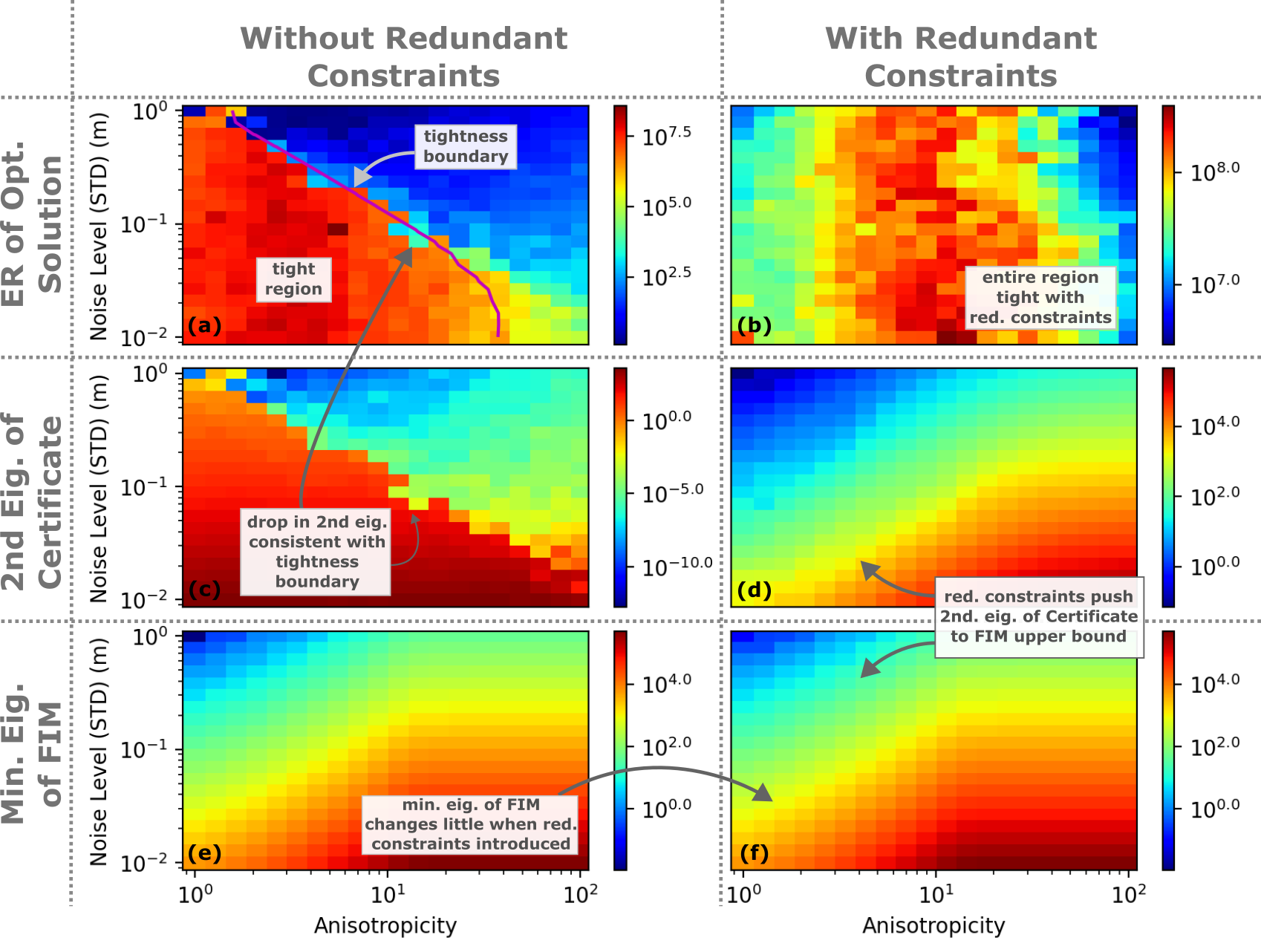}
	\caption{Extended study of the results in Figure~\ref{fig:ellipsoid_align} for the 50 landmark case. Heatmaps represent the value of each metric specified in the left-most column, both without (left column) and with (right column) redundant constraints. The presence of the redundant constraints pushes the second certificate eigenvalue closer to the minimum eigenvalue of the FIM, leading to a relaxation that is tight over a larger region of the parameter space. Note that there is a significant heatmap scale difference between (a) and (b) and between (c) and (d).}
	\label{fig:cert_eig_study}
\end{figure}

\subsubsection{Misaligned Uncertainty Ellipsoids}\label{sec:SimMisalign}

We now consider the effect of varying the \emph{orientation} of the uncertainty ellipsoids. In classical state estimation, it is known that high-quality estimates can still be obtained when measurements have high uncertainty in a particular direction, as long as these directions are not aligned. Leveraging our insight from Section~\ref{sec:Uncertainty}, we expect to see a similar trend with the tightness of the SDP relaxation.
 
Figure~\ref{fig:wahba_axis_bnd} demonstrates two experiments that vary uncertainty ellipsoid orientation in the same setting as for Figure~\ref{fig:ellipsoid_align}. Figure~\ref{fig:wahba_axis_bnd}(a)  shows results for random rotational perturbation of ellipsoid alignment while Figure~\ref{fig:wahba_axis_bnd}(b)  shows the effect of the `fanning out' of ray-aligned ellipsoids as the size of the landmark bounding cube is increased. 
\begin{figure}[!b]
	\centering
	\includegraphics[width=\columnwidth]{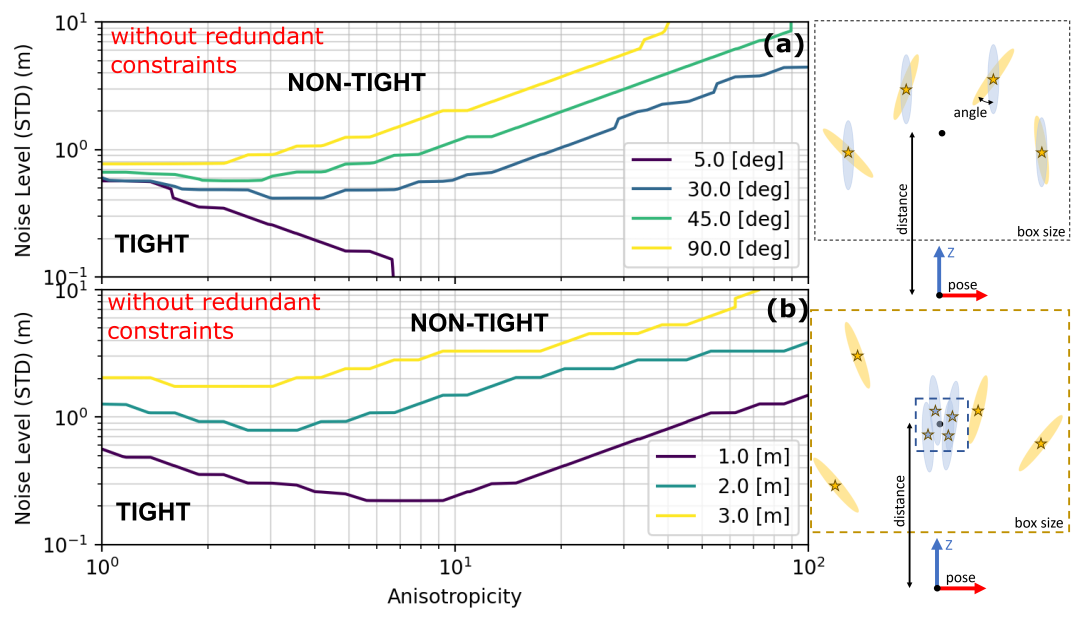}
	\caption{Tightness boundary for Wahba's problem instance with 20 landmarks observed. In (a), axes of maximum uncertainty are aligned to pose $z$-axis then perturbed by random angles with different standard deviations (see legend). In (b), maximum uncertainty axis is aligned to pose-landmark ray and size of the bounding box on the landmarks is varied (size in legend). In both cases, increased variation of uncertainty ellipsoid axes yields tighter relaxations as anisotropicity increases. Additionally, in both cases, the entire parameter space was tight when redundant constraints were used.}
	\label{fig:wahba_axis_bnd}
\end{figure}

In both cases, the trends coincide with our expectation: when the ellipsoids are \textit{aligned}, uncertainty is concentrated on a single axis and tightness is lost at a lower noise level. On the other hand, when the ellipsoids are highly \textit{misaligned}, uncertainty is not concentrated and the problems are tighter for much higher noise levels. 

Another interesting trend that we observe is that in most cases the tightness boundary begins to increase along the noise-level axis as anisotropicity increases. Because the axes are misaligned, high uncertainty of one measurement along a given direction, can be complemented by low certainty of a different measurement along the same direction. We posit that this causes the posterior uncertainty (and hence tightness) to be dominated by ellipsoid minor-axis uncertainty, which decreases as anistropicity increases. 

We did not include the results for Figure~\ref{fig:wahba_axis_bnd} with redundant constraints because all of the studied parameter values were tight.

\subsection{Wahba's Problem with Stereo-Camera Noise Model}\label{sec:SimWahbaStereo}

Levaraging our understanding from the previous section, we now turn our attention to Wahba's problem with measurements obtained from a less idealized, stereo-camera-based noise model. Both the measurement noise level and anisotropicity depend strongly on the distance between a pose and the observed landmarks (see \eqref{eqn:stereo_jac} at the end of Appendix~\ref{App:stereo}). Additionally, posterior uncertainty in stereo-camera-based localization depends on the `spread' of the landmarks being measured. As such, our experiments in this section explore tightness as a function of the distance to and size of the landmark bounding cube (see Figure~\ref{fig:stereo_redun}(a) for problem setup). The measurements and their associated covariances are drawn from the stereo-camera model with parameters given in Table~\ref{tbl:cam_params}.

Figure~\ref{fig:stereo_redun} (b) and (c) show the tightness boundaries for this setup without and with the redundant constraints given in \eqref{eqn:so3_constraints}. 
The figure shows results for Wahba's problem with different numbers of landmarks. As expected, the level of noise for which the relaxation is tight is inversely proportional to the average distance between pose and landmarks. On the other hand, increasing the spread (bounding cube) of the landmarks increases the tightness boundary, which is consistent with our previous observations for the ellipsoid model. 

Note that for this camera model, Wahba's problem is not very tight without redundant constraints: the tightness boundary when 50 landmarks are constrained to a unit cube is a range of approximately 2 m for a 0.24 m baseline camera. In constrast, the problem becomes considerably tighter when the redundant constraints are employed, with the maximum range of 100 m. This \rev{indicates} that redundant constraints are required for practical robotics applications. As we will see in Section~\ref{sec:OutdoorLoc}, it is feasible to use redundant constraints in this problem, while still maintaining computational efficiency. 

\begin{figure}[!t]
	\centering
	\includegraphics[width=\columnwidth]{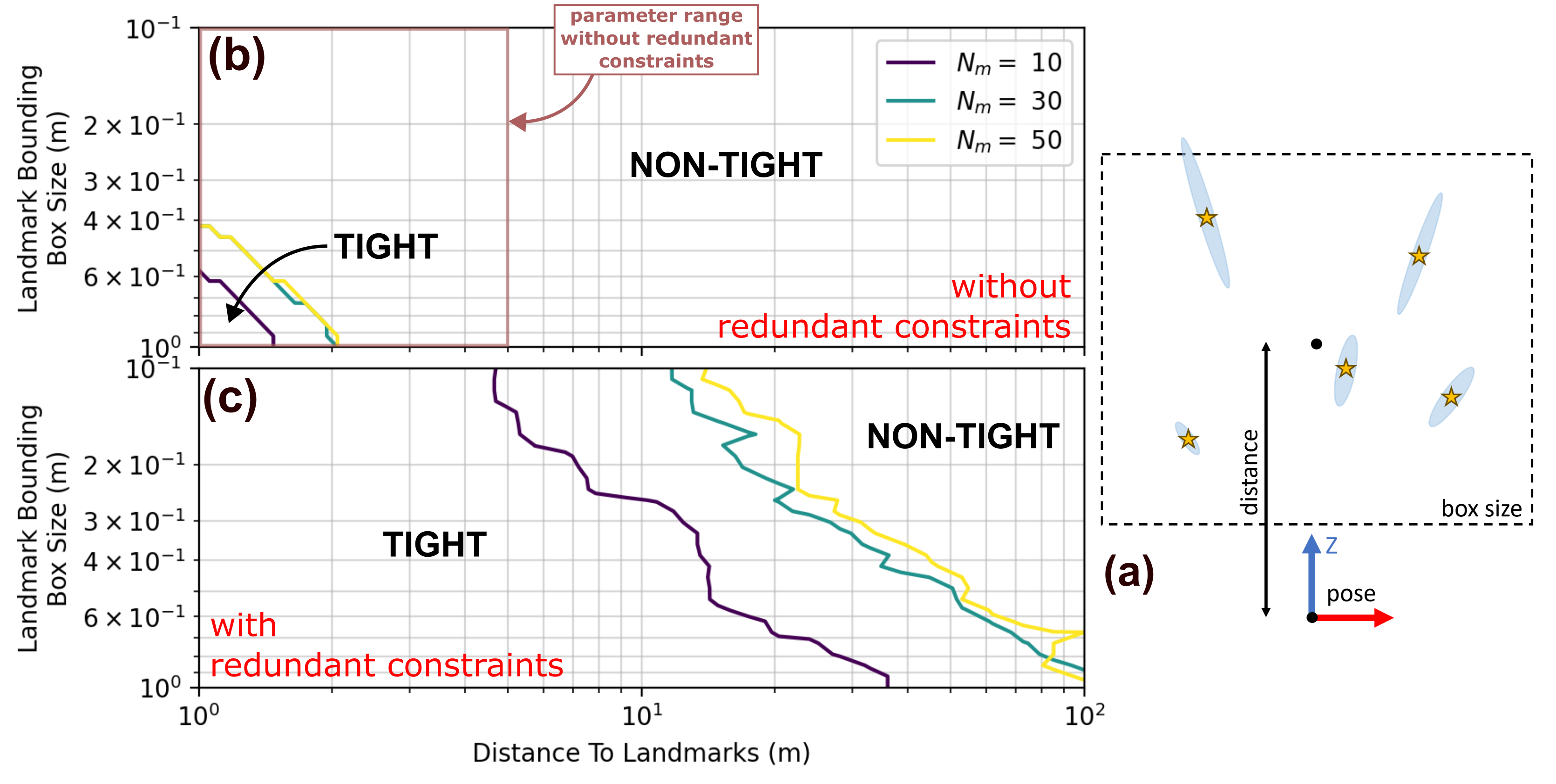}
	\caption{Investigation of tightness boundaries for Wahba's problem with a stereo-camera model without and with redundant constraints. A diagram of experiment setup is shown in (a). (b) shows the effect of increasing the number of landmarks without redundant constraints. (c) shows the boundaries for the same model and scenario, but with redundant constraints included. Note that the sampled distance parameter ranges are different between (a) and (b).}
	\label{fig:stereo_redun}
\end{figure}

\subsection{SLAM with Stereo-Camera Noise Model}\label{sec:SimStereoSLAM}

In Figure~\ref{fig:stereo_slam_box}, \rev{we show the effect of applying redundant constraints to a simplified SLAM problem} using the same simulation setup as in Figure~\ref{fig:stereo_redun}\rev{, with 20 landmarks and the landmark bounding box set to 1 m. We see that the application of the redundant constraints causes a considerable increase in the tightness (ER) of the problem in many instances. As in the localization results, tightness is eroded as the landmarks move further from the robot pose. We also observed that tightness is proportional to the number of landmarks observed.}

We note that although the redundant constraints are capable of tightening the problem, the runtime of these problems is prohibitive for real-time application. This is due to the size of the problem: for a single pose problem with 20 landmarks the dimension of the SDP variable is $133$ and the number of constraints required is $3343$. \rev{Even for a small SLAM problem such as this, the SDP runtime is approximately five seconds and} these numbers also increase quite rapidly as the number of poses increases. 

\begin{figure}[!t]
	\centering
	\includegraphics[width=\columnwidth]{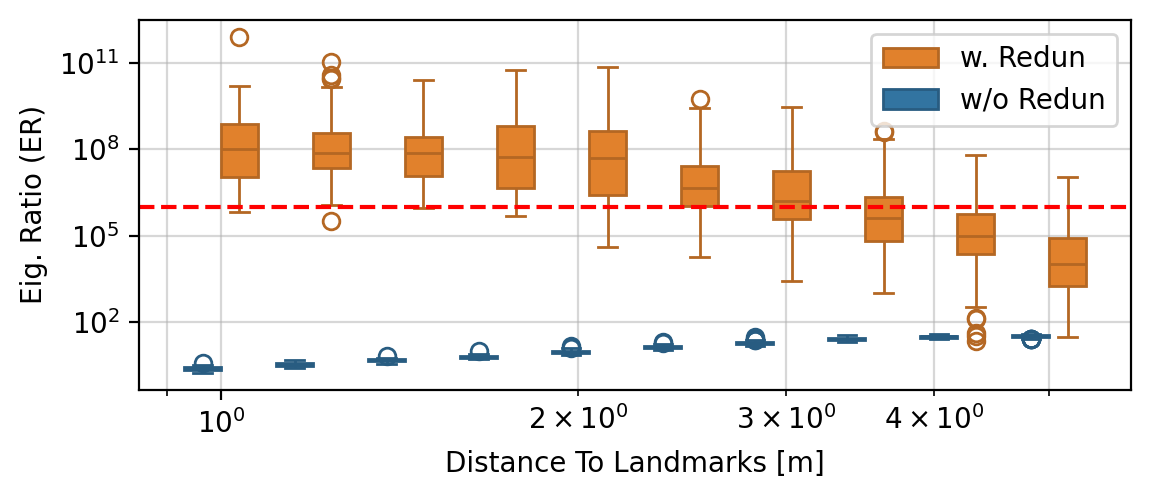}
	\caption{Investigation of tightness for a single-pose, stereo-camera, SLAM problem \rev{with and without redundant constraints (orange and blue, respectively). The dashed red line indicates the ER at which we consider a problem to be tight ($1\times10^6$). Each box represents 100 trials set up according to the diagram shown in Figure~\ref{fig:stereo_redun} (a), with 20 landmarks uniformly sampled from a 1 m bounding box centered at different distances from the pose ($x$-axis). The redundant constraints significantly improve the tightness (ER) across trials.}}
	\label{fig:stereo_slam_box}
\end{figure}

\section{Real-World Experiments}\label{sec:RealExp}
\subsection{Outdoor Stereo Localization}\label{sec:OutdoorLoc}

In the preceding sections, we have mentioned the fact that since matrix-weighted Wahba's problem has a small number of variables and constraints ($13$ and$31$, respectively), it is still feasible to use the SDP relaxation for post processing or even real-time applications. In this section, we apply our relaxation of matrix-weighted Wahba's problem (Problem \eqref{opt:Localize}) in the stereo-localization pipeline introduced in~\cite{gridsethKeepingEyeThings2022}, which uses a neural network to detect a set of features that are robust to seasonal and lighting conditions.

\subsubsection{Stereo-Localization Pipeline}

The full pipeline can be seen in Figure~\ref{fig:inthedark}(d). A neural network is used to detect features and provide descriptors for subsequent data association. The features are converted from 2D stereo keypoints to 3D keypoints, which are then used to estimate relative poses between keyframe stereo images in a stored map and corresponding stereo images from a `live run'. In~\cite{gridsethKeepingEyeThings2022}, the pose is estimated via random sample consensus followed by pose refinement with a scalar-weighted Singular Value Decomposition (SVD)\footnote{See~\cite{umeyamaLeastSquaresEstimationTransformation1991} for more details on this method.} approach.

\subsubsection{Modifications to Pose Estimation}

We replace the pose refinement block with a matrix-weighted optimization (i.e., Wahba's problem). A stereo-camera model is used to compute the inverse covariances of the 3D keypoint measurements (see Section~\ref{App:stereo} for details), which are then used as the matrix weights. Since the pipeline does not make use of relative-pose measurements (i.e., IMU data), each pose can be solved separately. 

We solved the matrix-weighted pose-refinement optimization with both a local solver and global solver. The local optimization was performed using a Gauss-Newton solver over the $\mbox{SE}(3)$ Lie group in an off-the-shelf framework called Theseus~\cite{pinedaTheseusLibraryDifferentiable2022}. Tolerances (relative and absolute) were set to $1\times10^{-10}$ with 200 as the maximum number of iterations. The maximum number of iterations recorded was 160. Theseus was initialized using the best pose estimate from RANSAC, as is common in practice. 

The problem was solved globally via the SDP relaxation of Problem \eqref{opt:Localize} \emph{with} the redundant constraints given in \eqref{eqn:so3_constraints}. The cost matrix was computed as shown in Section~\ref{App:LocCost}. For each pose, we used CVXPY~\cite{diamondCVXPYPythonEmbeddedModeling} with Mosek~\cite{mosek} to solve the SDP. The solution was extracted by selecting the column corresponding to the homogenizing variable in the solution matrix. The relevant interior-point tolerances for Mosek were also set to $1\times10^{-10}$ with the maximum iterations set to 1000 (though the number of iterations did not exceed 30).

All approaches were implemented in PyTorch. Note that the entire pipeline up to the pose refinement step (including RANSAC with SVD to find inliers) was identical for both solvers pipelines. On average, approximately 542 inlier 3D measurements were returned by RANSAC for the refinement step.  

\subsubsection{Dataset and Results}
%\color{red}

\begin{figure}[]
	\centering
	\includegraphics[width=\columnwidth]{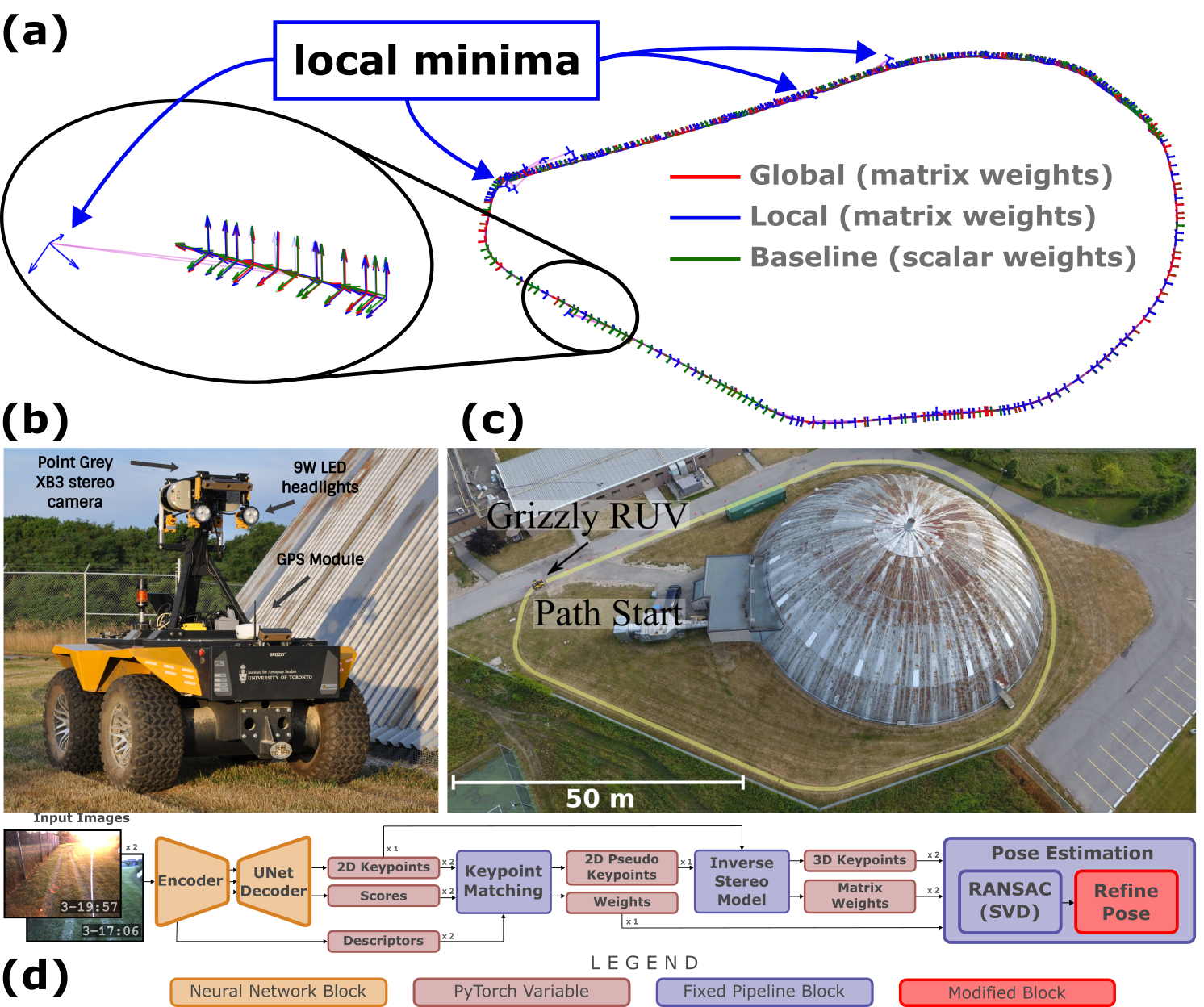}
	\caption{Overview of the setup and results of the outdoor stereo localization experiment introduced in Section~\ref{sec:OutdoorLoc} (map run: 2, live run: 16). (a) shows the estimated trajectory for all three methods and highlights the local minima that appear when using the Local method. (b) shows the robot (Grizzly RUV) and camera (Point Grey XB3) that were used. (c) shows an aerial view of the localization track and (d) shows the stereo-localization pipeline (with modified blocks in \textcolor{red}{red}).}
	\label{fig:inthedark}
\end{figure}

We tested the pipeline on runs 2, 11, 16, 17, 23, 28, and 35 from the `In-the-Dark' dataset, which was used for training and testing in~\cite{gridsethKeepingEyeThings2022}.\footnote{Dataset is available at \url{http://asrl.utias.utoronto.ca/datasets/2020-vtr-dataset/}. The selected runs correspond to the hold-out test runs from~\cite{gridsethKeepingEyeThings2022}.} The number of poses in each run varies, ranging between 886 and 3634, but the path taken in all runs is shown in Figure~\ref{fig:inthedark} (c). 

We localize between all pairs of the seven runs (21 localizations total). Each run can be used as either the `map' or the `live' run since all are equiped with ground truth. Localization was performed for the Baseline (SVD), Local (Theseus), and Global (SDP) solvers by running the pipeline with each on an NVIDIA Tesla V100 DGXS GPU with a Intel Xeon 2.20GHz CPU. We present the aggregate results of the analysis in terms of average time per pose, longitudinal root-mean-square error (RMSE), latitudinal RMSE, and heading RMSE for each of the pipelines in Table~\ref{tbl:inthedark}. 

\rev{We observe} that, \rev{in terms of accuracy,} the Local solver performs significantly worse than the Baseline and Global methods. Further investigation revealed that this was because, for several poses throughout the runs, the Local solver converged to egregious local minima. An example of such a minimum can be seen in Figure~\ref{fig:inthedark}(a). As such, we have added a `filtered' row to Table~\ref{tbl:inthedark}, which provides the Local solver results when pose estimates with heading error larger than 3 deg are removed. Even when filtered, the Local solver still has the largest RMSE values, possibly due to less salient local minima.

It is interesting that the (scalar-weighted) Baseline approach outperforms the (matrix-weighted) approaches in terms of RMSE and compute time. However, we note that the `ground truth' trajectories for each run were computed by minimizing a reprojection-based cost (see~\cite{patonBridgingAppearanceGap2016a}) and are subject to additional error, which is likely on the order of the difference between the Baseline and Global solvers ($\leq0.005$ m and $\leq0.05$ deg). Further, the fact that the neural network was trained \emph{using} the Baseline approach likely contributes to the performance difference.

Finally, we call attention to the fact that, in terms of speed, the Global SDP approach ran only about 2.3 times slower than the (closed-form) Baseline solution. Given the speed (8.19 Hz) and accuracy of the SDP solver, the authors argue that it could be used for online estimation, especially if implemented in a more performant language than Python (such as C++).

\begin{table}
	\label{tbl:inthedark}
	\centering
	%\color{red}
	\caption{Aggregate Results Across Runs for In-The-Dark Dataset}
	\begin{tblr}{lcrrrrr}
		\hline
		Pipeline & Weights & {Avg. Time\\Per Pose} & {Long.\\RMSE\\(m)} &  {Lat.\\RMSE\\(m)} &   {Head.\\RMSE\\(deg)} \\
		\hline
		Baseline     &  scalar  &    0.055 & 0.025 & 0.013 & 0.249 \\
		Global       &  matrix  &    0.121 & 0.033 & 0.021 & 0.335 \\
		Local        &  matrix  &    0.118 & 0.175 & 0.153 & 5.415 \\
		{Local\\(filtered)}& matrix &    0.118 & 0.037 & 0.035 & 0.598 \\
		\hline
	\end{tblr}
\end{table}

\subsection{Stereo SLAM in a Controlled Environment}\label{sec:stereoslam_starry}

In this section, we test matrix-weighted SLAM on the `Starry Night' dataset~\cite{barfoot2011state}, which provides stereo-camera measurements of a set of known landmark locations (with known data association). The parameters of the stereo-camera model are shown in Table~\ref{tbl:cam_params}, with parameter symbols consistent with those described in Appendix~\ref{App:stereo}.

\begin{table}[H]
	%\color{red}
	\caption{Ground-Truth Camera Parameters}\label{tbl:cam_params}
	\centering
    \begin{tabular}{ |l|c|c|c|c|c|c|c| }
    \hline
    Parameter& $b $ & $f_u$ & $f_v$ & $c_u$ & $c_v$ & $\sigma_u$ &  $\sigma_v$ \\ \hline
	Units& m & $\frac{\mbox{pix}}{\mbox{m}}$ & $\frac{\mbox{pix}}{\mbox{m}}$ & pix & pix & pix & pix \\
	\hline
    Values & 0.24 & 484.5  & 484.5  & 0.0  &  0.0 & 6.32 &  11.45 \\
	\hline
    \end{tabular} 
\end{table}

To assess tightness of matrix-weighted SLAM with stereo-image measurements, we randomly select 10 poses from the dataset and solve the SDP relaxation with and without redundant constraints. We add relative-pose measurements between subsequent poses based on the ground-truth data and perturb these measurements by a controlled amount of (isotropic) noise allowing to assess the effect of these additional measurements. All measurements and matrix-weights were defined as explained in Section~\ref{sec:Background} and the related sections of the Appendix.

The results are shown in Figure~\ref{fig:d3_slam_box}. At all times, the distance between the pose and the landmarks is within 1.75 m.  Without redundant constraints, the problem is not tight across the noise levels considered. On the other hand, with redundant constraints, the problem remains tight until the relative-pose measurement noise exceeds approximately 0.06 m and 0.06 rad, reflecting tightness up to a reasonable level of noise. 

\begin{figure}[!t]
	\centering
	\includegraphics[width=\columnwidth]{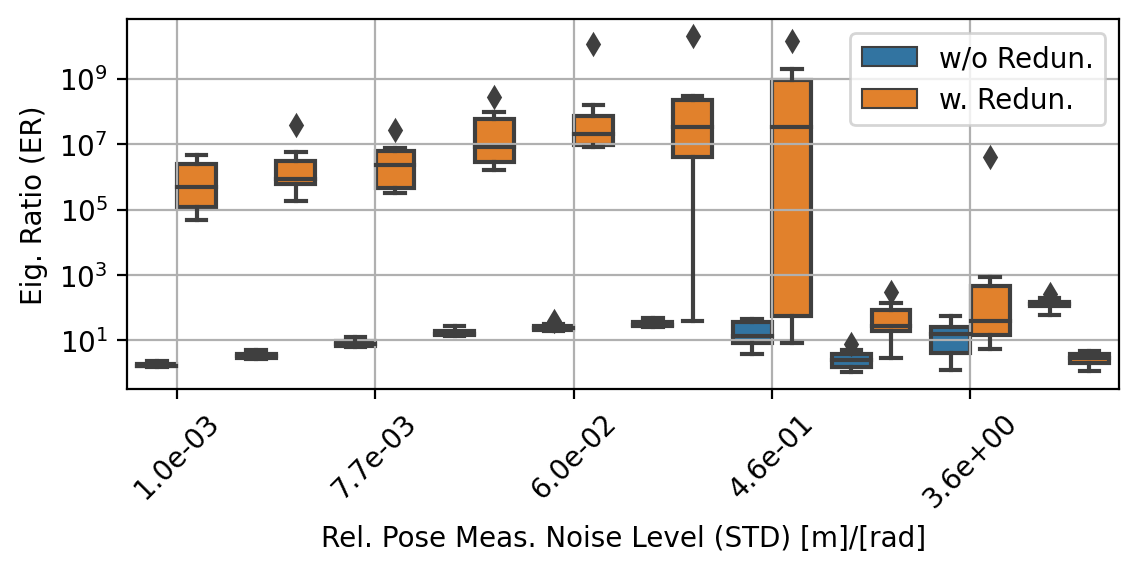}
	\caption{ER  results with and without redundant constraints for SLAM on random selections of 10 poses from the ``Starry Night'' dataset with relative-pose measurements between subsequent poses. 10 trials (random selection of poses) were performed at each noise level. Without redundant constraints (blue), the SLAM problem is not tight for any of the tested noise levels. With redundant constraints (orange) the problem is tight for reasonable noise levels.}
	\label{fig:d3_slam_box}
\end{figure}

Figure~\ref{fig:slam_local_min} shows an example of a local and global minimum for a 10-pose SLAM problem using measurements from the ``Starry Night'' dataset. Relative-pose measurements between subsequent poses were perturbed by Gaussian noise with standard deviation of $5.0 \times 10^{-2}$ m and $5.0 \times 10^{-2}$ rad in both rotation and translation. The local minimum was the result of running a Gauss-Newton (GN) solver from a randomly selected, poor initialization point (gradient and cost converged to $8.87\times10^{-12}$ and $1.13\times10^3$, respectively) while the global solution was found by solving the matrix-weighted SLAM SDP relaxation with redundant constraints ($\mbox{ER}$ and cost converged to $2.43\times10^6$ and $9.36$, respectively).\footnote{Note that for good initializations, it was confirmed that the GN solver converged to the same cost and solution as the SDP relaxation.} 

It is important to note that we have restricted these experiments to a low number of poses because of the \emph{necessary} addition of variables and constraints, which makes the matrix-weighted SLAM problem intractable for medium-to-large-scale problems. However, we believe that the analysis and results presented herein are valuable to the robotics community and demonstrate that certifiable methods can be extended to this problem.

\section{Conclusions}\label{sec:Conclusions}
\subsection{Discussion}
We have shown that inclusion of matrix weights in state-estimation problems can have profound implications on the tightness of their semidefinite relaxations. In the case of Wahba's problem, matrix weights can decrease the noise level for which a given problem instance is tight to well below those found in practice. For SLAM, the introduction of matrix weights leads to a fundamental change in the formulation of the problem and the resulting formulation is not tight even for very low noise levels (without redundant constraints). 

We have established a key connection between the posterior distribution of a state estimate and the dual or certificate matrix. We have also explored the relationship between the noise distribution of measurements, the posterior distribution, and the tightness of the semidefinite relaxation. Namely, anisotropicity in the underlying noise model results in non-tight relaxations when uncertainty is aligned in a given direction. This effect can be counteracted (to some extent) by increasing the number and variety of measurements in the problem.

One of the goals of our analysis was to determine whether redundant constraints are necessary for the problems considered herein. Based on the results in Sections~\ref{sec:SimWahbaStereo} and~\ref{sec:SimStereoSLAM}, redundant constraints seem to be necessary to achieve reliable robustness to noise in stereo-camera-based applications. In the case of Wahba's problem, our results in Section~\ref{sec:OutdoorLoc} show that, \rev{even with redundant constraints, directly solving the SDP can be fast enough to be used in real applications}. Moreover, we have shown that even when initialized well, local methods can converge to egregious local minima. On the other hand, for matrix-weighted SLAM, the addition of even a sparse set of redundant constraints seems to be prohibitive for online use, as shown in Sections~\ref{sec:SimStereoSLAM} and~\ref{sec:stereoslam_starry}.

\subsection{Future Work}\label{sec:Future}

The state-of-the-art, large-scale, certifiable perception methods in robotics concentrate on cases that do not require redundant constraints~\cite{rosenSESyncCertifiablyCorrect2019,brialesCartanSyncFastGlobal2017}. However, even a relatively small change to the noise model (e.g., introducing anisotropicity) seems to necessitate such constraints. Therefore, we posit that a crucial area of development for certifiable algorithms in robotics involves maintaining computational speed even with these additional constraints. 

The exploration of amendments to the \emph{Riemannian Staircase} and \emph{Burer-Monteiro} techniques to accommodate redundant constraints is one potential avenue of future work. This would involve addressing potential bottlenecks during the certification step due to a costly search for optimal Lagrange multipliers. 

Another key avenue, which has proved essential in the past, is to exploit the sparsity of the SDP with redundant constraints as shown in~\cite{zhengChordalFactorwidthDecompositions2021}. In conjunction with the design of sparsity-promoting redundant constraints, this approach may yield the speed that is required for more general, online, large-scale certifiable solvers.

\section{Acknowledgments}

The authors would like to thank Natural Sciences and Engineering Research Council of Canada (NSERC) for their generous support of this work. The authors would also like to thank Matt Giamou and Dave Rosen for their useful insights and suggestions.
 
\bibliographystyle{IEEEtrans}
\bibliography{MatrixWeightCert,Software}

\appendix

\subsection{Proof of Lemma~\ref{lem:FisherInfo}}\label{App:lemma1Proof}
\begin{proof}
\rev{Taking inspiration from~\cite{nocedalNumericalOptimization2006}, we use a second-order Taylor approximation of the Lagrangian of the QCQP~\eqref{opt:QCQP} to prove our result. This Lagrangian} can be expressed as follows~\cite{boydConvexOptimization2004}:
\begin{equation*}
	\mathcal{L}(\bm{z},\bm{\lambda}, \rho) = \bm{z}^T \left(\bm{Q} + \rho\bm{A}_0 +\sum\limits_{i=1}^{N_c} \lambda_i \bm{A}_i\right) \bm{z} - \rho.
\end{equation*}
Recalling the definition of the certificate matrix from the dual problem, \eqref{opt:Dual}, we have
\begin{equation*}
	\mathcal{L}(\bm{z},\bm{\lambda}, \rho) = \bm{z}^T \bm{H}(\bm{\lambda}, \rho) \bm{z} - \rho.
\end{equation*}

Let $\bm{x} = \bm{x}^*+\delta\bm{x} \in \mathcal{U}$ be a perturbation of the optimal solution, $\bm{x}^*$, of Problem \eqref{opt:UnconstrainedMAP}. Since, by assumption, the objective functions are locally equal, we have $\bm{z}^* = \bm{\ell}(x^*)$.  Let $\bm{z}= \bm{\ell}(x)$. Since $\bm{\ell}$ is smooth, we have the following perturbation to $\bm{z}^*$:
\begin{equation}\label{eqn:z_pert}
	\delta\bm{z} = \bm{z} - \bm{z}^* = \bm{L} \delta\bm{x} + O(\delta\bm{x}),
\end{equation}
where $\delta\bm{x}=\bm{x} - \bm{x}^*$ and $O(\cdot)$ denotes Bachmann-Landau (``big-o'') notation. Since $\bm{z}$ is in the feasible set of Problem \eqref{opt:QCQP}, the Lagrangian at $\bm{z}$ is equal to the objective of Problem \eqref{opt:QCQP} and, by assumption, we have
\begin{equation*}
	\mathcal{L}(\bm{z},\bm{\lambda}, \rho) =\mathcal{L}(\bm{\ell}(\bm{x}),\bm{\lambda}, \rho)=\bm{\ell}(\bm{x})^T\bm{Q}\bm{\ell}(\bm{x}) = -\log(p(\bm{x} \vert \bm{\mathcal{D}})).
\end{equation*}

We now proceed by considering second-order Taylor expansions of the Lagrangian about $\bm{z}^*$.\footnote{As in~\cite[Theorem 12.5]{nocedalNumericalOptimization2006} we leave the Lagrange multipliers fixed} Letting  $\bm{H}= \bm{H}(\bm{\lambda}^*, \rho^*)$ to simplify notation, we have
\begin{align*}
	\mathcal{L}(\bm{z},\bm{\lambda}^*, \rho^*) & =\bm{z}^T \bm{H} \bm{z} - \rho^* \\ 
	&= (\bm{z}^*+\delta\bm{z})^T \bm{H} (\bm{z}^*+\delta\bm{z}) - \rho^* \\ 
	&= -\rho^* + \delta\bm{z}^T \bm{H}\delta\bm{z}
\end{align*}
where the third line follows from the first-order necessary optimality conditions ($\bm{H} \bm{z}^*=\bm{0}$). Applying \eqref{eqn:z_pert}, we get
\begin{equation*}
	-\log(p(\bm{x} \vert \bm{\mathcal{D}})) = -\rho^* + \delta\bm{x}^T \bm{L}^T\bm{H}\bm{L}\delta\bm{x} + O(\delta\bm{x}^2),
\end{equation*}
where $c$ is a constant. Since the FIM is exactly the Hessian of $-\log(p(\bm{x} \vert \bm{\mathcal{D}}))$ evaluated at $\delta\bm{x} = \bm{0}$, we have the result:
\begin{equation}
	\bm{\Sigma}^{-1} = \bm{L}^T\bm{H}\bm{L}.
\end{equation}
\end{proof}

\subsection{Proof of Proposition~\ref{prop:eig_bounds}}\label{App:lem2Proof}

\begin{proof}
By assumption, we have $\bm{\Sigma}^{-1} = \bm{L}^T\bm{H}\bm{L}$. Without loss of generality, assume that the homogenizing variable of $\bm{z}$ is the last element in the vector. Since this element is always equal to one, the Jacobian of the feasible set mapping is given by $\bm{L}^T = \begin{bmatrix}
	\bar{\bm{L}}^T & \bm{0}^T
\end{bmatrix}$. It follows that $\bm{\Sigma}^{-1} = \bar{\bm{L}}^T\bar{\bm{H}}\bar{\bm{L}}$.

Since $\bm{\ell}$ is injective, $\bar{\bm{L}}$ has full column rank. It follows that\footnote{See proof of Corollary 2.4.4 in~\cite{golubMatrixComputations2013}.}
\begin{equation}\label{eqn:normSandwich}
	\rev{s_{\min}}(\bar{\bm{L}}) \Vert\bm{v}\Vert \leq  \Vert\bar{\bm{L}}\bm{v}\Vert \leq \rev{s_{\max}}(\bar{\bm{L}})\Vert\bm{v}\Vert,
\end{equation}
where $\Vert\cdot\Vert$ denotes the Euclidean norm, $\rev{s_{\min}}(\bar{\bm{L}})$ is the minimum singular value of $\bm{L}$, and $\rev{s_{\max}}(\bar{\bm{L}})$ is the maximum singular value of $\bm{L}$. We proceed using the Raleigh quotient characterization of eigenvalues, noting that both $\bm{\Sigma}^{-1}$ and $\bar{\bm{H}}$ are \rev{real, symmetric} matrices.
\begin{align*}
	\sigma_{\min}(\bar{\bm{H}}) &= \min\limits_{\bm{y}\in\mathbb{R}^n} \frac{\bm{y}^T \bar{\bm{H}} \bm{y}}{\Vert\bm{y}\Vert^2} 
	\leq \min\limits_{\substack{\bm{y}=\bar{\bm{L}}\bm{v}\\\bm{v}\in\mathbb{R}^p}} \frac{\bm{y}^T \bar{\bm{H}} \bm{y}}{\Vert\bm{y}\Vert^2} \\&= \min\limits_{\bm{v}\in\mathbb{R}^p} \frac{\bm{v}^T \bar{\bm{L}}^T\bar{\bm{H}}\bar{\bm{L}} \bm{v}}{\Vert\bar{\bm{L}}\bm{v}\Vert^2 }  
	\leq \min\limits_{\bm{v}\in\mathbb{R}^p} \frac{\bm{v}^T \bar{\bm{L}}^T\bar{\bm{H}}\bar{\bm{L}} \bm{v}}{\rev{s_{\min}}(\bar{\bm{L}})^2\Vert\bm{v}\Vert^2 } 
	\\&= \frac{\sigma_{\min}(\bm{\Sigma}^{-1})}{\rev{s_{\min}}(\bar{\bm{L}})^2},
\end{align*}
where the first inequality follows from the fact that we are restricting the feasible set of the optimization and the second follows from \eqref{eqn:normSandwich}. Finally, note that $\sigma_{\min}(\bar{\bm{L}}) = \sigma_{\min}(\bm{L})$.

\end{proof}

\rev{
\subsection{Degeneracy of the FIM and Tightness}\label{app:fim-degeneracy}
\begin{proposition}
	Assume that the setting of Proposition~\ref{prop:eig_bounds} holds and strict complementarity holds for Problem~\eqref{opt:SDP}. If the FIM is degenerate (i.e., $\sigma_{\min}(\bm{\Sigma}^{-1}) = 0$), then the SDP relaxation becomes non-tight. That is, either $\bm{H}$ has negative eigenvalues or the optimal primal SDP solution has rank higher than one.
\end{proposition}
\begin{proof}
	Let $\bm{z}_0 = \begin{bmatrix} \bm{y}_0^T &1 \end{bmatrix}^T$ be the globally optimal solution to the QCQP, with associated certificate matrix $\bm{H}$. By definition, we have $\bm{H}\bm{z}_0 = \bm{0}$. It can be shown that this implies that $\bm{h} = -\bar{\bm{H}}\bm{y}_0$.
	
	By assumption, we have that $\sigma_{\min}(\bm{\Sigma}^{-1})=0$. Thus, Proposition~\ref{prop:eig_bounds} implies that $\sigma_{\min}(\bar{\bm{H}})\leq 0$. If $\sigma_{\min}(\bar{\bm{H}})<0$ then $\sigma_{\min}(\bm{H}) < 0$ and the relaxation cannot be tight (since $\bm{z}_0$ is globally optimal). 
	
	On the other hand, if $\sigma_{\min}(\bm{H}) \geq 0$ then $\sigma_{\min}(\bar{\bm{H}}) \geq 0$ and, by Proposition~\ref{prop:eig_bounds}, we have $\sigma_{\min}(\bar{\bm{H}}) = 0$. It follows that there exists a vector $\bm{y}_1$ such that $\bm{\bar{H}}\bm{y}_1=\bm{0}$. 
	Let $\bm{z}_1 = \begin{bmatrix} \bm{y}_1^T &0 \end{bmatrix}^T$ and note that,
	\begin{equation*}
		\bm{H}\bm{z}_1 = \begin{bmatrix}
			\bm{\bar{H}}\bm{y}_1 \\ \bm{h}^T\bm{y}_1
		\end{bmatrix} = \begin{bmatrix}
			\bm{\bar{H}}\bm{y}_1 \\ -\bm{y}_0^T\bm{\bar{H}}\bm{y}_1
		\end{bmatrix} = \bm{0}.
	\end{equation*}
	Since $\bm{z}_0$ and $\bm{z}_1$ are linearly independent, $\bm{H}$ has two minimum eigenvalues at zero. By strict complementarity, the corresponding primal SDP solution will have rank equal to two and is therefore not tight.
\end{proof}
\begin{remark}
	From the proof, we see that degeneracy of the FIM implies that the Lagrangian has the same value (to first order) for any vector, $\bm{z}_0 + \mu\bm{z}_1 $, as long as $\mu\in\mathbb{R}$ is small. A further implication is that at the global solution, $\bm{z}_0$, there is a feasible direction, $\bm{z}_1$, which does not lead to an increase in the objective. This implies that the original QCQP has symmetries in its solution (i.e., its solution is not unique). It has been shown that such symmetries in the original QCQP solution cause higher rank SDP relaxation solutions, corresponding to the convex hull of the lifted rank-one solutions \cite{brialesCertifiablyGloballyOptimal2018}.
\end{remark}
}

\subsection{Cost Function Elements}\label{App:LocCost}

In this section, we show how to determine the QCQP cost matrices for the localization and SLAM problems given in Sections~\ref{sec:Localization} and~\ref{sec:SLAM}. Throughout this section, we use the facts that $ \tr{\bm{A}\bm{B}\bm{C}} = (\bm{C}^T\otimes\bm{A}) \vect{\bm{B}}$, $ \bm{A} = 1\otimes\bm{A} $, and $ (\bm{A}\otimes\bm{B})(\bm{C}\otimes\bm{D}) = (\bm{AC}\otimes\bm{BD}) $\cite{magnusMatrixDifferentialCalculus2019}. 

\subsubsection{Landmark Measurements}

Consider a single cost element corresponding to an edge in $\EdgeSet_m $:
\begin{align*}
	J_{ik}=&(\tilde{\bm{m}}_i^{ki} - \bm{C}_i\bm{m}_{0}^{k0} + \bm{t}_i)^T \bm{W}_k (\tilde{\bm{m}}_i^{ki} - \bm{C}_i\bm{m}_{0}^{k0} + \bm{t}_i) \\
	=& \tilde{\bm{m}}_i^{ki^T}\bm{W}_k\tilde{\bm{m}}_i^{ki} w^2 - 2 w\tilde{\bm{m}}_i^{ki^T}\bm{W}_k\bm{C}_i\bm{m}_{0}^{k0}  \\
	&+ 2 w\tilde{\bm{m}}_i^{ki^T}\bm{W}_k\bm{t}_i -2\bm{m}_{0}^{k0^T}\bm{C}_i^T\bm{W}_k\bm{t}_i\\
	& + \bm{m}_{0}^{k0^T}\bm{C}_i^T\bm{W}_k\bm{C}_i\bm{m}_{0}^{k0} + \bm{t}_i^T\bm{W}_k\bm{t}_i,
\end{align*}
Re-organizing terms, we see that this cost element can be written as $J_{ik}~=~\bm{x}^T\bm{Q}_{ik}\bm{x}_i$, where $ \bm{x}_i^T = \begin{bmatrix} \bm{c}_i^T &  \bm{t}_i^T & w \end{bmatrix} $ and $ \bm{c}_i=\vect{\bm{C}_i} $.
\begin{equation*}
	\resizebox{0.98\columnwidth}{!}{$\bm{Q}_{ik}=\begin{bmatrix}
			\bm{m}_{0}^{k0}\bm{m}_{0}^{k0^T} \otimes \bm{W}_k  & -\bm{m}_{0}^{k0}\otimes\bm{W}_k & -\bm{m}_{0}^{k0}\otimes\bm{W}_k\tilde{\bm{m}}_i^{ki}\\
			 -\bm{m}_{0}^{k0^T}\otimes\bm{W}_k & \bm{W}_k & \bm{W}_k\tilde{\bm{m}}_i^{ki} \\
			-\bm{m}_{0}^{k0^T}\otimes\tilde{\bm{m}}_i^{ki^T}\bm{W}_k & \tilde{\bm{m}}_i^{ki^T}\bm{W}_k & \tilde{\bm{m}}_i^{ki^T}\bm{W}_k\tilde{\bm{m}}_i^{ki}
		\end{bmatrix}.$}
\end{equation*}

\subsubsection{Relative-Pose Measurements}

Similarly, each edge in $\EdgeSet_p$ represents a cost element of the following form:
\begin{equation}
	J_{ij} = \frac{1}{\sigma^2_{ij}} \left\Vert \tilde{\bm{C}}_{ij}\bm{C}_j - \bm{C}_i\right\Vert_F^2 + \frac{1}{\tau^2_{ij}} \left\Vert \tilde{\bm{t}}^{ji}_{i} - \tilde{\bm{C}}_{ij}\bm{t}_j + \bm{t}_i \right\Vert_2^2.
\end{equation}
As before, we collect the relevant variables,
\begin{equation*}
	\bm{x}_{ij}^T = \begin{bmatrix} \bm{c}_i^T&\bm{c}_j^T  & \bm{t}_i^T & \bm{t}_j^T & w\end{bmatrix},
\end{equation*}
allowing us to write the cost element as
\begin{gather*}
	J_{ij}= \bm{x}_{ij}^T\bm{Q}_{ij}\bm{x}_{ij}, \\
	\bm{Q}_{ij} = \begin{bmatrix}
		\frac{1}{\sigma^2_{ij}}\bm{Q}_{r,ij} & \bm{0}\\
		\bm{0}&\frac{1}{\tau^2_{ij}}\bm{Q}_{t,ij}
	\end{bmatrix},
\end{gather*}
with
\begin{gather*}
	\bm{Q}_{r,ij}=\begin{bmatrix}
		\bm{I}  &  -\bm{I}\otimes\tilde{\bm{C}}_{ij}^T\\
		-\bm{I}\otimes\tilde{\bm{C}}_{ij} & \bm{I}
	\end{bmatrix}, \\ 
	\bm{Q}_{t,ij}=\begin{bmatrix}
		\bm{I}  &  -\tilde{\bm{C}}_{ij} & \tilde{\bm{t}}_{i}^{ji}\\
		-\tilde{\bm{C}}_{ij}^T & \bm{I} & -\tilde{\bm{C}}_{ij}\tilde{\bm{t}}_{i}^{ji}\\
		\tilde{\bm{t}}_{i}^{ji^T} &  -\tilde{\bm{t}}_{i}^{ji^T}\tilde{\bm{C}}_{ij}^T & \tilde{\bm{t}}_{i}^{ji^T} \tilde{\bm{t}}_{i}^{ji}
	\end{bmatrix}.
\end{gather*}

The blocks of these cost elements can be permuted according to the variable ordering defined for a given problem and the final cost matrix is obtained by summing all cost elements.

\subsection{Stereo-Camera Model}\label{App:stereo}

In this section, we seek to convert pixel measurements from left- and right-rectified stereo images to Euclidean point measurements. We also seek to determine an appropriate model of the uncertainty of the measurement in the Euclidean space. We assume that left and right stereo frames have their $z$-axes aligned with the viewing direction and their other axes coincident. We also assume that they are separated by baseline distance $ b\in \mathbb{R} $ and that the camera frame is coincident with the left camera.  

Consider a 3D point expressed in the world frame, $ \bm{x}_w = \begin{bmatrix} x_w & y_w & z_w  \end{bmatrix}^T $ representing a feature that we wish to track. This point can be expressed in the left camera frame as  $ \bm{x}_c = \bm{C}_{cw} \bm{x}_w + \bm{t}^{wc}_c $, where $ \bm{C}_{cw} \in \mbox{SO}(3)$ is the rotation matrix from the world to the camera frame and $ \bm{t}^{wc}_c \in \mathbb{R}^3 $ is vector from the world frame origin to the camera frame origin expressed in the camera frame.

Assuming a pinhole stereo-camera, the resulting pixel measurements can be expressed as follows:
\begin{equation}
	\begin{bmatrix}
		p_{ul}\\p_{vl}\\p_{ur}\\p_{vr}
	\end{bmatrix} = \frac{1}{z_c}\begin{bmatrix}
		f_{u} & 0 & c_{u} & 0\\
		0 & f_{v} & c_{v} & 0\\
		f_{u} & 0 & c_{u} & -b f_{u}\\
		0 & f_{v} & c_{v} & 0
	\end{bmatrix} \begin{bmatrix}
	\bm{x}_c \\ 1
	\end{bmatrix} + \bm{\epsilon}_{p},
\end{equation}
where $ u $ and $ v $ subscripts represent horizontal and vertical pixel directions, respectively, $ l $ and $ r $ subscripts represent left and right cameras, respectively, $ p_{ij} $ represents the pixel measurement, $ f_{i} $ represents the focal length parameter, and $ c_{i} $ represents the camera centre parameter for direction $ i $ of camera $ j $. The variable $ \bm{\epsilon}_{p} $ represents noise on the pixel measurements and is assumed to have zero-mean, Gaussian distribution,
$ \bm{\epsilon}_{p} \sim \mathcal{N}(\bm{0}, \bm{\Sigma}_p) $,
with covariance matrix $ \bm{\Sigma}_p = \diag{\sigma_u^2 ,\sigma_v^2,\sigma_u^2, \sigma_v^2} $, where $ \sigma_i $ represents the standard deviation of the noise. 

We define the \textit{disparity}, $ d $, of a given feature as the horizontal difference between the feature's position in the right and left images in terms of pixels: $ d = p_{ul} - p_{ur} $. We define the intermediate, Gaussian-distributed measurement, 
\begin{equation}
	\bm{y} = \begin{bmatrix}p_{ul}&p_{vl}&d\end{bmatrix}^T \sim \mathcal{N}(\bm{\mu}_y, \bm{\Sigma}_y),
\end{equation}
where the mean and covariance matrices are given by
\begin{equation}
	\bm{\mu}_y = \begin{bmatrix}
		\frac{1}{z_c}f_u x_c + c_u \\ \frac{1}{z_c}f_v y_c + c_v \\ \frac{1}{z_c}f_u b
	\end{bmatrix},\quad \bm{\Sigma}_y = \begin{bmatrix}
		\sigma_u^2 & 0 & \sigma_u^2\\
		0 & \sigma_v^2 & 0 \\
		\sigma_u^2 & 0 & 2\sigma_u^2
	\end{bmatrix}.
\end{equation}
Given this measurement, we can use the (known) intrinsic camera parameters to generate Euclidean \textit{pseudo-measurements}, $ \hat{\bm{x}} = \begin{bmatrix}\hat{x}_c &\hat{y}_c & \hat{z}_c \end{bmatrix}^T $, of the (unknown) feature locations via the following mapping:
\begin{equation}
	\bm{g}^{-1}:~(\mathbb{R}^2\times\mathbb{R}_+) \rightarrow \mathbb{R}^3, \quad \mbox{s.t.}~\hat{\bm{x}} = \bm{g}^{-1}(\bm{y})=b\begin{bmatrix}
	 	\frac{p_{ul}-c_u}{d} \\ \frac{p_{vl}-c_v}{d} \\ \frac{f_u}{d} 
	 \end{bmatrix}.
\end{equation}
Although this transformation is non-linear, we make the assumption that the distribution of $ \hat{\bm{x}} $ remains Gaussian. To approximate the covariance of the pseudo-measurement, we map the intermediate measurement covariance through the linearized Jacobian of this transformation:
\begin{equation}
	\bm{G} = \frac{\partial \bm{g}^{-1}(\bm{y})}{\partial \bm{y}} = b\begin{bmatrix}
		\frac{1}{d} & 0 & -\frac{p_{ul} - c_u}{d^2} \\
		0 & \frac{f_u}{f_v d} & -\frac{f_u}{f_v }\frac{(p_{vl} - c_v)}{d^2} \\
		0 & 0 & -\frac{f_u}{d^2}
	\end{bmatrix}.
\end{equation}
Typically, such a linearization is performed about a prior belief -- as in a Kalman filter -- or previous iterate -- as in iteratively re-weighted least squares --  of the variables involved. However, in the global-optimization context, neither of these options are available and we choose to linearize about the measurement itself. 

Therefore, the noise model of the pseudo-measurement can be expressed as
\begin{equation}
	 \hat{\bm{x}} = \bm{x}_c + \bm{\epsilon}_x, ~ \bm{\epsilon}_x \sim \mathcal{N}(\bm{0}, \bm{\Sigma}_x),
\end{equation}
where $ \bm{\epsilon}_x $ represents an approximately Gaussian-distributed variable with zero mean and covariance given by $ \bm{\Sigma}_x = \bm{G}\bm{\Sigma}_y \bm{G}^T $.

It is instructive to consider the following alternate formulation of the Jacobian:
\begin{equation}\label{eqn:stereo_jac}
	\bm{G} =\begin{bmatrix}
		\frac{\hat{z}_c}{f_u} & 0 & -\frac{\hat{x}_c}{f_u }\frac{\hat{z}_c}{b} \\
		0 & \frac{\hat{z}_c}{f_v } & -\frac{\hat{y}_c}{f_u }\frac{\hat{z}_c}{b} \\
		0 & 0 & -\frac{\hat{z}_c}{f_u}\frac{\hat{z}_c}{b}
	\end{bmatrix}.
\end{equation}
We see that the Jacobian matrix scales linearly with the $z$-axis coordinate, $\hat{z}_c$, meaning that the variance of the Euclidean measurement scales quadratically with this variable. Ignoring the off-diagonal terms, we note that the anistropicity of the measurement is approximately proportional to $\frac{\hat{z}_c}{b}$.

\subsection{SDP Stability of Matrix-Weighted Localization}\label{App:SDPStability}

In this section, we prove that Problem \eqref{opt:Localize} enjoys the property of `SDP stability' established in~\cite{cifuentesLocalStabilitySemidefinite2022}. That is, the convex relaxation of Problem \eqref{opt:Localize} is tight whenever the set of measurements has sufficiently low noise. In particular, we connect the SDP stability to the well-known condition that the observed landmarks are not coplanar, which has been shown to lead to a unique solution for point-set regression~\cite{arunLeastSquaresFittingTwo1987}. Such results have been given for other state-estimation problems~\cite{rosenSESyncCertifiablyCorrect2019, tianDistributedCertifiablyCorrect2021} and our proof closely follows the development given in~\cite{wiseCertifiablyOptimalMonocular2020}.

\begin{definition}[Noise-free Measurements]
	Given the setting of Problem \eqref{opt:Localize}, we define the set of \emph{noise-free measurements} for a given set of poses, $\left\{(\bar{\bm{C}}_{i}, \bar{\bm{t}}_i) ~ \forall i \in \VertSetP \right\}$, as
	\begin{equation}
		\left\{\bar{\bm{m}}_i^{ki} = \bar{\bm{C}}_{i} \bm{m}_{0}^{k0} - \bar{\bm{t}}_i,~ \forall (i,k)\in \EdgeSet_m\right\},
	\end{equation}
	and collect all such measurements in a vector, $\bar{\bm{m}}$.
\end{definition}

\begin{theorem}[SDP-Stability of Matrix-Weighted Localization]
Consider the setting of Problem \eqref{opt:Localize} with positive definite weighting matrices ($ \bm{W}_{ik} \succ 0$) and suppose the set of landmarks observed from a given pose are not coplanar. Let $\tilde{\bm{m}}$ be a vector containing all of the pose-landmark measurements for the problem and let $\bar{\bm{m}}$ be the set of noise-free measurements associated with the globally optimal solution of the problem with the same ordering as $\tilde{\bm{m}}$. Then, there exists some $\epsilon > 0$ such that if $\Vert \tilde{\bm{m}} - \bar{\bm{m}} \Vert^2 < \epsilon $, the SDP relaxation of the problem is tight (strong duality holds) and the global minimizer can be recovered from the SDP solution.
\end{theorem}
\begin{proof}
To prove the theorem, we show that the conditions of Theorem 3.9 from~\cite{cifuentesLocalStabilitySemidefinite2022} are satisfied when the assumptions of our theorem hold. In the following, we treat the measurements, $\tilde{\bm{m}}$, as a set of \emph{parameters} on which the cost of the problem depends. Since there are no relative-pose measurements in Problem \eqref{opt:Localize}, it is separable and we can prove the result for a single-pose problem (fixed pose $i$) without loss of generality.
The conditions of Theorem 3.9 are as follows:
\begin{enumerate}
	\item The cost and constraints are quadratic in the variables. \label{thm1:cond1}
	\item The cost depends continuously on the set of parameters, $\tilde{\bm{m}}$. \label{thm1:cond2}
	\item When constructed with $\bar{\bm{m}}$ rather than $\tilde{\bm{m}}$, the cost becomes \emph{strictly convex}, with the same unconstrained minimum as the original problem.\label{thm1:cond3}
	\item The Abadie Constraint Qualification (ACQ) holds for the feasible set at the global solution. \label{thm1:cond4}
\end{enumerate}
Let $\bm{y} = \begin{bmatrix}\vect{\bm{C}_i}^T & \bm{t}_i^T \end{bmatrix}^T \in\mathbb{R}^{12}$ be the vectorized form of the variables.
The cost function for pose $i$ in Problem \eqref{opt:Localize} is given by 
\begin{equation}\label{eqn:LocCost}
	f_i(\tilde{\bm{m}},\bm{y}) = \sum\limits_{(i,k)\in\EdgeSet_m} \bm{e}_{ik}^T \bm{W}_{ik} \bm{e}_{ik}.
\end{equation}
We can write $\bm{e}_{ik}$ as an affine function of $\bm{y}$
\begin{equation}
	\bm{e}_{ik} = \tilde{\bm{m}}_i^{ki} - \bm{A}_k \bm{y}, ~\forall ~(i,k)\in \EdgeSet_m,
\end{equation}
where $\bm{A}_k = \begin{bmatrix} \bm{m}_{0}^{k0^T}\otimes\bm{I}& -\bm{I} \end{bmatrix}\in\mathbb{R}^{3\times12}$. The cost function becomes 
\begin{equation}
	f_i(\tilde{\bm{m}},\bm{y}) = \bm{y}^T \bm{Q} \bm{y} + \bm{b}^T \bm{y} + c, 
\end{equation}
where
\begin{gather}
	\bm{Q} = \sum\limits_{(i,k)\in\EdgeSet_m} \bm{A}_k^T \bm{W}_{ik} \bm{A}_k, \\
	\bm{b} = -2\sum\limits_{(i,k)\in\EdgeSet_m} \bm{W}_{ik}\tilde{\bm{m}}_i^{ik}, \\
	c=\sum\limits_{(i,k)\in\EdgeSet_m} \tilde{\bm{m}}_i^{ik^T}\bm{W}_{ik}\tilde{\bm{m}}_i^{ik}.
\end{gather}
We see immediately that the cost function depends quadratically on the variables, $\bm{y}$, and quadratically (thus continuously) on the parameters, $\tilde{\bm{m}}$. Together with the fact that the $\mbox{O}(3)$ constraints are quadratic, conditions~\ref{thm1:cond1}) and~\ref{thm1:cond2}) are satisfied.

Now, let $\bar{\bm{y}}$ represent the vectorized form of the global solution and consider the cost function, $f_i(\bar{\bm{m}},\bm{y})$, constructed using the noise-free measurements, $\bar{\bm{m}}$, associated with the global solution. Since $f_i(\bar{\bm{m}},\bm{y})$ is a sum of quadratic forms, we have that $f_i(\bar{\bm{m}},\bm{y})\geq 0$. Moreover, its global minimizer is $\bar{\bm{y}} $ since $f_i(\bar{\bm{m}},\bar{\bm{y}})= 0$. 

To conclude that this minimizer is unique, we must show \emph{strict} convexity of $f_i(\bar{\bm{m}},\bm{y})$, which is implied if $\nabla_{\bm{y}}^2 f_i(\bar{\bm{m}},\bm{y}) = \bm{Q} $ is strictly positive definite~\cite{boydConvexOptimization2004}. Note that since the weights are positive definite ($\bm{W}_{ik}\succ 0 $) we already have that $\bm{Q} \succeq 0$. To show positive definiteness, it remains to show that $\bm{Q}$ is full-rank. To this end, note that we can write
\begin{equation}
	\bm{Q} = \bm{A}^T \bm{W} \bm{A},
\end{equation}
where
\begin{gather}
	\bm{W} = \diag{\bm{W}_{i1},\ldots,\bm{W}_{iN}},\\
	\bm{A} = \begin{bmatrix}
		 \bm{A}_1 \\ \vdots \\ \bm{A}_N
	\end{bmatrix} = \begin{bmatrix}
		 \bm{m}_{0}^{10^T} & 1 \\
		 \vdots& \vdots\\
		 \bm{m}_{0}^{N0^T} & 1
	\end{bmatrix}\otimes \bm{I}_3,
\end{gather}
and we have re-indexed the landmarks from $1$ to $N$ for convenience. We also define
\begin{equation*}
	\bm{B} = \begin{bmatrix}
		\bm{m}_{0}^{10^T} & 1 \\
		\vdots&\vdots\\
		\bm{m}_{0}^{N0^T} & 1
	\end{bmatrix}.
\end{equation*}
Since $\bm{W}$ is full rank, we are required to show that $\rank{\bm{A}} = 12$. 
Recall that the singular values  of a Kronecker product, $\bm{A}\otimes\bm{B}$, is given by $\left\{ \mu\lambda,~ \forall \mu \in \sigma(\bm{A}),~ \lambda \in \sigma(\bm{B})\right\}$, where $\sigma()$ denotes the set of singular values.
Therefore, $\rank{\bm{A}} = 12$ is equivalent to the fact that $\rank{\bm{B}}=4$. We can subtract the first row from the remaining rows of $\bm{B}$ without changing its row rank. Thus,
\begin{equation*}
	\rank{\bm{B}} = \rank{\begin{bmatrix}
			\bm{m}_{0}^{10^T} & 1 \\
			\bm{D} & \bm{0}
		\end{bmatrix}}, ~ \bm{D} = \begin{bmatrix}
		(\bm{m}_{0}^{20}-\bm{m}_{0}^{10})^T  \\
		\vdots\\
		(\bm{m}_{0}^{N0}-\bm{m}_{0}^{10})^T 
	\end{bmatrix}.\\	
\end{equation*}
Since the landmarks are not coplanar by assumption, $\rank{\bm{D}}=3$ and, since $\rank{\bm{B}} = \rank{\bm{D}} + 1$, condition~\ref{thm1:cond3} holds.
 
Finally, the ACQ holds at any feasible point by the argument given in~\cite{cifuentesLocalStabilitySemidefinite2022}, giving the final condition of Theorem 3.9.
\end{proof}

We conclude this section with two remarks. First, the assumption that the weight matrices are positive definite (i.e., not degenerate) is not strictly required as long as $\rank{\bm{A}^T \bm{W} \bm{A}}=12$ is guaranteed by sufficient measurements. As mentioned above, degenerate weight matrices were encountered in~\cite{brialesConvexGlobal3D2017} when measurements of lines and planes were considered. Second, the extension of this proof to include relative-pose measurements should be \rev{possible} since it has already been proved when \emph{only} relative-pose measurements are considered (at least for the case of a weakly connected measurement graph)~\cite{rosenSESyncCertifiablyCorrect2019}.

\subsection{Other Metrics of Tightness}\label{App:otherMetrics}

In other works, it is sometimes the case that a `relative gap' is used as a metric for evaluating tightness of a convex relaxation. This gap is typically defined as follows:
\begin{equation*}
	\mbox{gap} = \frac{p(\bm{x}_r)-d^*}{1+d^*},
\end{equation*}
where $p(\bm{x}_r)$ represents the primal cost of the closest feasible, rank-1 solution (rounded solution) and $d^*$ represents the optimal SDP cost (equivalently, the dual cost). The one in the denominator keeps the metric stable when $d^*$ is low. 

In Figure~\ref{fig:compare_er_gap}, we compare this metric with the ER, our metric of choice throughout the paper. We note that the corresponding boundaries are very close, except when anistropicity is large. When redundant constraints were used, both metrics showed that the relaxation was tight for all parameter values.

\begin{figure}[!t]
	\centering
	\includegraphics[width=\columnwidth]{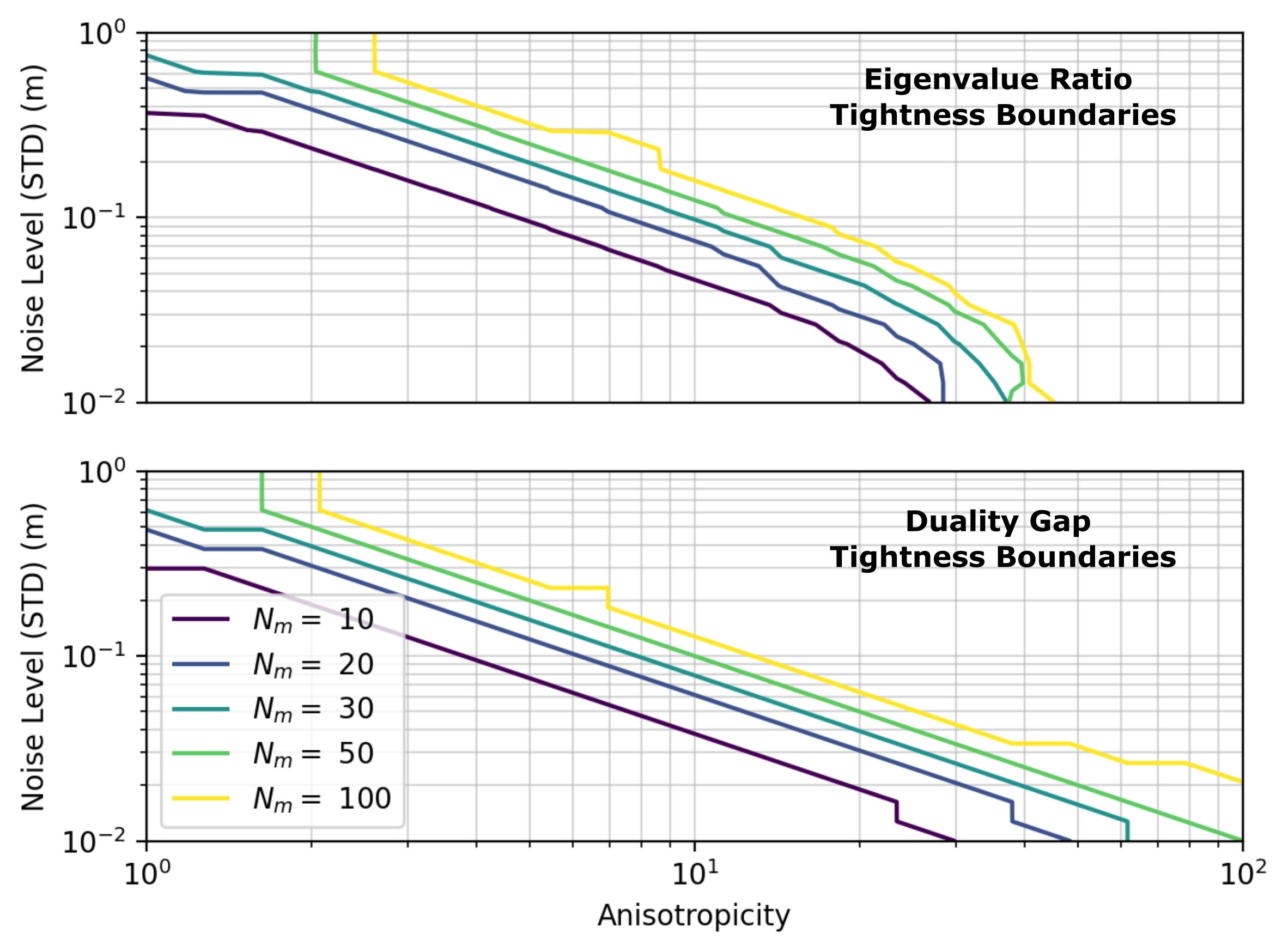}
	\caption{Comparison of ER and duality gap tightness boundary contours for the results in Figure~\ref{fig:ellipsoid_align} (without redundant constraints). The ER boundary corresponds to an ER of $1\times10^{-6}$ whereas the duality gap boundary corresponds to a relative gap of $1\times10^{-10}$.}
	\label{fig:compare_er_gap}
\end{figure}

\begin{figure}[!t]
	\centering
	\includegraphics[width=\columnwidth]{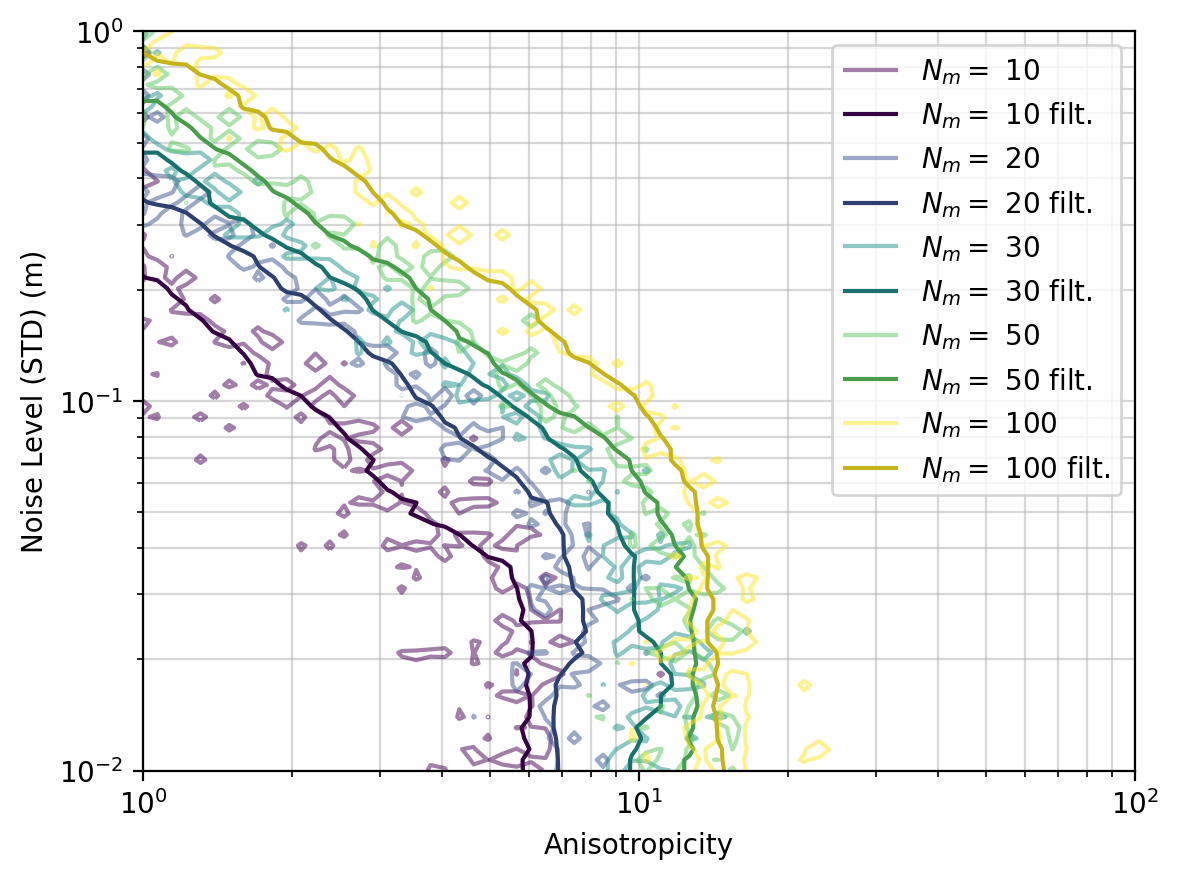}
	\caption{Comparison of tightness boundary contours with and without median filter for the results in Figure~\ref{fig:ellipsoid_align}. Darker contours (marked ``filt.'' in legend) represent the smoothed contours that were used in the body text.}
	\label{fig:filter}
\end{figure}

\subsection{Tightness Boundary Smoothing}\label{App:Smoothing}

As mentioned in the main body, we plot tightness boundaries based on the parameters for which the minimum ER across all trials passes a threshold. These contours were plotted using the `contour' function from PyPlot.  However, it was initially found that the resulting contour plots were quite noisy and difficult to interpret even when the number of trials was increased considerably. To filter the data, the median of minimum ER  values was taken over an $N$-by-$N$ block in the parameter space (centered on a given parameter value) and the result was used to generate the smoothed contours. It was found that $N=7$ was sufficient to ensure adequate smoothing. A comparison of the contours on the raw minimum ER  data versus the filtered data is shown in Figure~\ref{fig:filter}, in which the data corresponds to the data in Figure~\ref{fig:ellipsoid_align} in the main text. 

\begin{IEEEbiography}[{\includegraphics[width=1in,height=1.25in,clip,keepaspectratio]{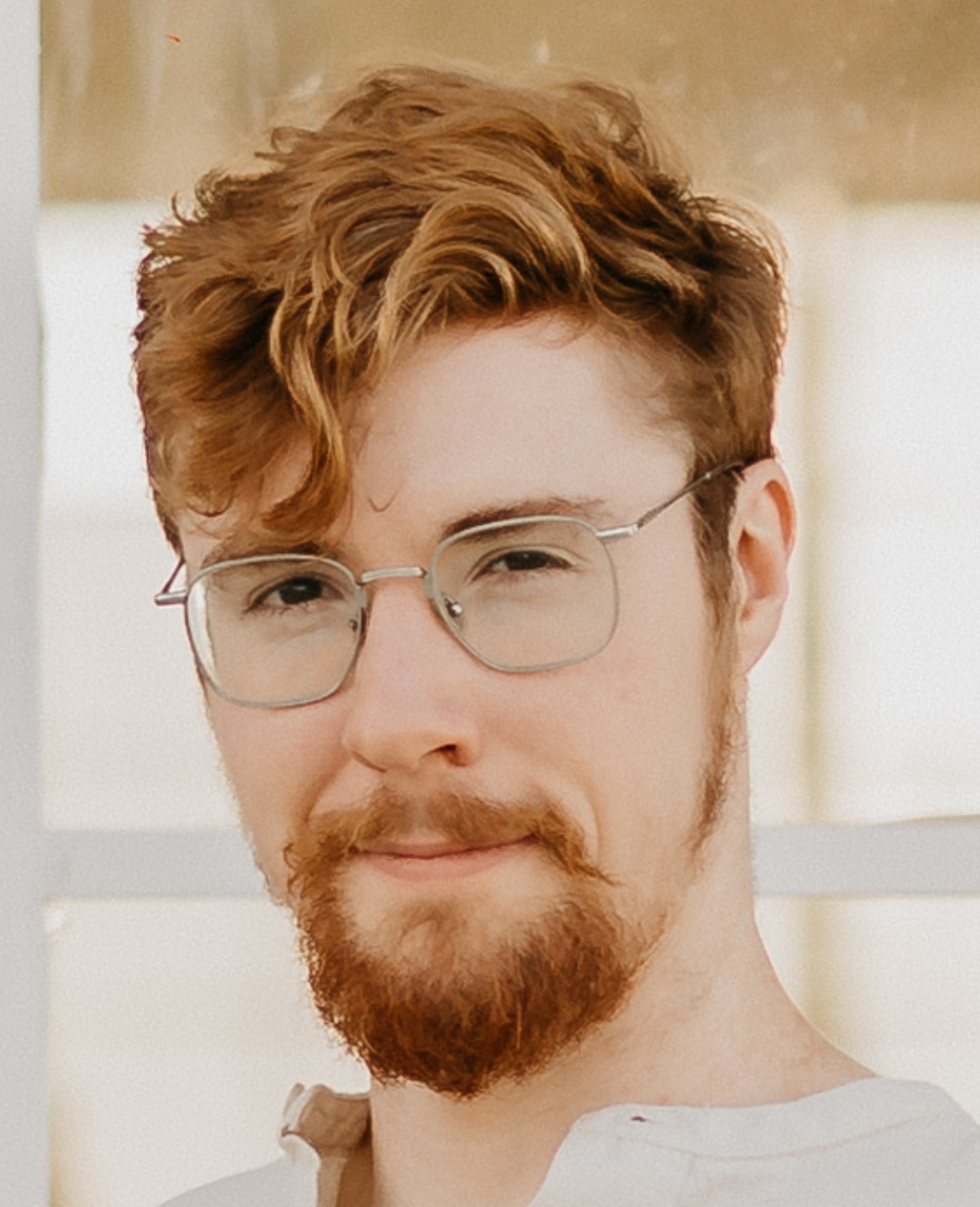}}]{Connor Holmes} received a B.A.Sc. degree in Engineering Science and M.A.Sc. degree in Electrical and Computer Engineering from the University of Toronto in 2014 and 2016, respectively. From 2016 until 2021, he worked as an Guidance Navigation and Controls Engineer at MDA Space. Since 2021, he has been pursuing a Ph.D. at the University of Toronto Robotics Institute. Connor's research interests include the application of convex optimization in robotics, particularly for certification of state-estimation algorithms.
\end{IEEEbiography}

\begin{IEEEbiography}[{\includegraphics[width=1in,height=1.25in,clip,keepaspectratio]{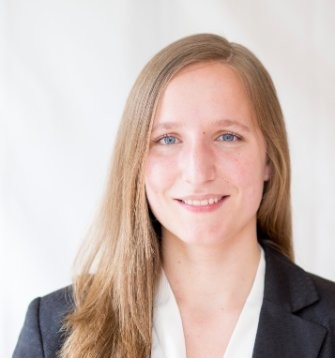}}]{Frederike Dümbgen} received the B.Sc. and M.Sc. degrees in Mechanical Engineering, with minor in Computational Science and Engineering, from École Polytechnique Fédérale de Lausanne (EPFL) in 2013 and 2016, respectively, conducting her Master’s thesis at ETH Zürich. She obtained her Ph.D. degree in computer and communication sciences from EPFL in 2021 and worked as a post-doctoral researcher at the Robotics Institute of University of Toronto, Canada, from 2022 to 2024. As of May 2024, she is a researcher in the WILLOW team, affiliated with Inria and ENS, PSL University, Paris. Her research interests lie in the areas of estimation and advanced optimization for robotics.   
\end{IEEEbiography}

\begin{IEEEbiography}[{\includegraphics[width=1in,height=1.25in,clip,keepaspectratio]{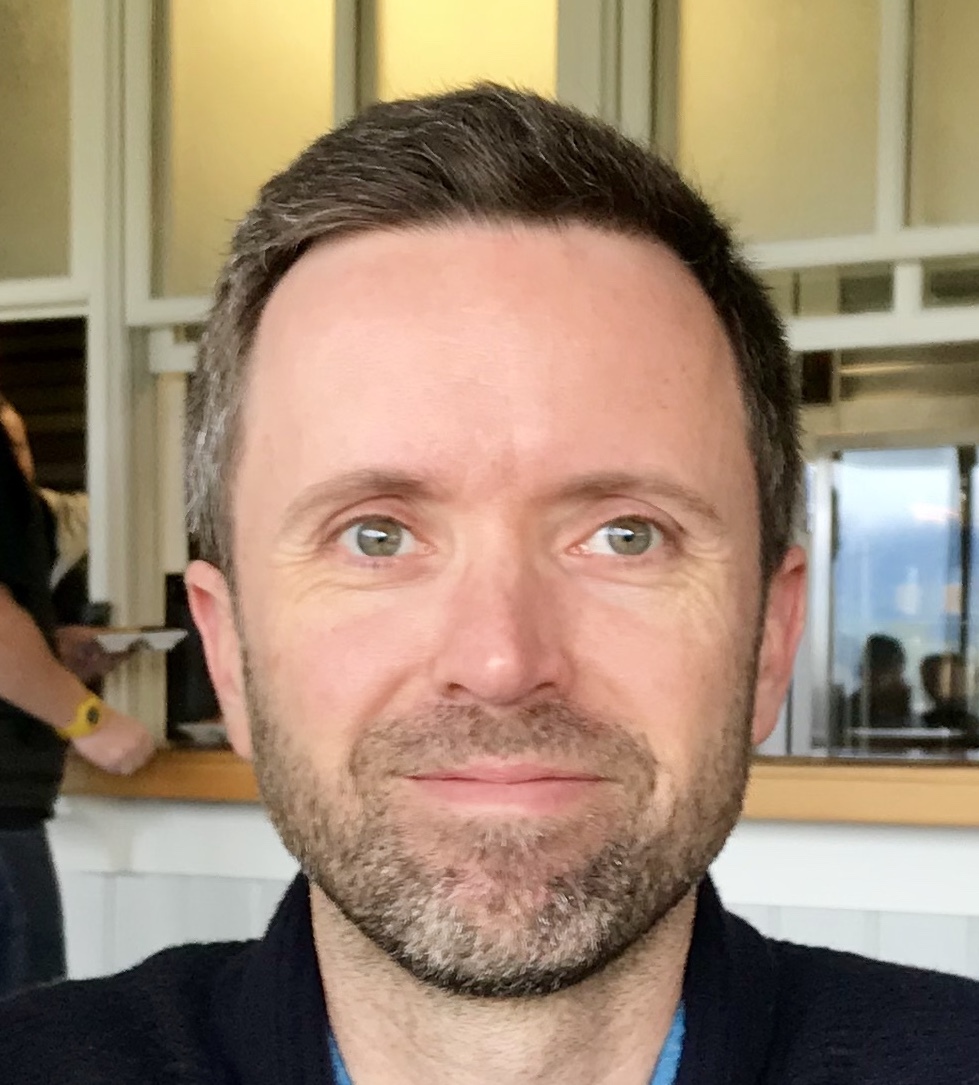}}]{Timothy D. Barfoot} received the B.A.Sc. degree in Engineering Science from University of Toronto, Toronto, ON, Canada, in 1997 and the Ph.D. degree in aerospace science and engineering from University of Toronto, in 2002. He is a Professor with the University of Toronto Robotics Institute, Toronto, ON, Canada. He works in the areas of guidance, navigation, and control of autonomous systems for a variety of applications. He is interested in developing methods to allow robotic systems to operate over long periods of time in large-scale, unstructured, three-dimensional environments, using rich onboard sensing (e.g., cameras and laser rangefinders) and computation.
\end{IEEEbiography}

\vfill

\end{document}